\documentclass[times, review, 10pt]{elsarticle}

\usepackage{amsmath}
\usepackage{amsfonts}
\usepackage{amssymb}
\usepackage{amsthm}
\usepackage{mathrsfs}
\usepackage{graphicx}
\usepackage{booktabs}
\usepackage{bbm}
\usepackage{multirow}
\usepackage{subcaption}
\usepackage{algorithm}
\usepackage{algpseudocode}
\usepackage{lineno}   
\usepackage{setspace}
\usepackage{hyperref}
\usepackage{geometry}
\newtheorem{theorem}{Theorem}

\biboptions{sort&compress}

\journal{Journal Name}  

\begin{document}

\begin{frontmatter}

\title{\textbf{A Heterogeneous Mixture of Experts Framework for Interpretable Machine Learning}}


\author[isi1]{Soham Chatterjee}
\ead{amisohambolchi.chat@gmail.com} \corref{cor1}    
\cortext[cor1]{Corresponding author}

\author[isi1]{Rwitobroto Dey}
\ead{dey.rwitobroto.0612@gmail.com} 

\author[isi2]{Smarajit Bose}
\ead{smarajit@isical.ac.in}  

\affiliation[isi1]{organization={Indian Statistical Institute},
            addressline={203 B.T. Road},
            city={Kolkata},
            postcode={700108},
            country={India}}
\affiliation[isi2]{organization={Interdisciplinary Statistical Research Unit,
                   Indian Statistical Institute},
            addressline={203 B.T. Road},
            city={Kolkata},
            postcode={700108},
            country={India}}

\begin{abstract}
Mixture-of-Experts (MoE) models provide a flexible framework for partitioning complex prediction problems into simpler local learning tasks through an input-dependent gating mechanism. Existing interpretable MoE approaches, such as Mixture of Decision Trees (MoDT), achieve transparency by employing homogeneous decision-tree experts, but this restricts the model to a single inductive bias across all regions of the feature space. We extend the MoDT framework by introducing heterogeneous expert families comprising decision trees, linear support vector machines, and quadratic discriminant analysis under a common probabilistic gating mechanism. To ensure coherent likelihood-based inference, non-probabilistic experts are calibrated to produce conditional class probabilities, allowing parameter estimation within the generalized Expectation-Maximization framework of MoDT. We further establish theoretical monotone ascent guarantees for the proposed heterogeneous gating updates, providing a justification for the optimization procedure. Experiments on a diverse collection of synthetic and real-world benchmark datasets demonstrate that the proposed framework adaptively specializes experts according to local data geometry, yielding interpretable expert assignments while achieving predictive performance competitive with homogeneous MoDT and Random Forests. The proposed approach combines interpretability, adaptive inductive bias selection, and probabilistic coherence within a unified mixture-of-experts framework.
\end{abstract}

\end{frontmatter}


\section{Introduction}
\label{sec:introduction}

Mixture-of-experts (MoE) models provide a principled framework for combining multiple predictive models through an input-dependent gating mechanism that dynamically allocates observations to specialized experts. Introduced by Jacobs et al.~\cite{jacobs1991} and subsequently generalized through hierarchical formulations by Jordan and Jacobs~\cite{jordan1994}, MoE architectures decompose complex prediction problems into simpler local subproblems, allowing different experts to specialize in different regions of the input space. This adaptive specialization mechanism offers an appealing compromise between model flexibility and statistical efficiency, and has become an influential paradigm in machine learning, ensemble learning, and probabilistic modeling~\cite{dietterich2000,mclachlan2000,gan2025,zhou2012}.

From a statistical learning perspective, the success of MoE models stems from their ability to exploit heterogeneous structure within the data. Rather than relying on a single global hypothesis class, the gating mechanism allows different experts to dominate in regions where their inductive biases are most appropriate. This idea is closely aligned with a central principle of statistical learning theory: no single model class is uniformly optimal across all learning problems, and predictive performance depends critically on matching model assumptions to the underlying data structure~\cite{hastie2009,isl2021,vapnik1998,shalev2014}.

Despite their flexibility, classical MoE models have often faced criticism regarding interpretability. Many modern MoE architectures employ highly expressive neural-network gating mechanisms and complex expert models whose decision processes are difficult to understand, reflecting a broader challenge in interpretable machine learning \cite{molnar2022,guidotti2018,rudin2022}. This challenge has motivated a growing body of research on interpretable machine learning and interpretable mixtures of experts\cite{ribeiro2016lime,doshivelez2017}. Examples include MOET~\cite{vasic2019}, which employs tree-based experts in reinforcement learning, Interpretable Mixture of Experts~\cite{ismail2022}, which seeks globally understandable expert decompositions, InterpretCC~\cite{swamy2024}, which focuses on user-centric interpretability, and recent intrinsically interpretable MoE architectures~\cite{yang2025}. These works demonstrate increasing interest in preserving the adaptive specialization advantages of MoE models while maintaining transparency and explainability.

Among existing interpretable MoE approaches, the Mixture of Decision Trees (MoDT) framework proposed by Br\"{u}ggenj\"{u}rgen et al.~\cite{bru} is particularly attractive. MoDT combines a linear softmax gating mechanism with shallow decision-tree experts, thereby preserving the probabilistic structure of classical MoE models while providing interpretable rule-based decision paths. By restricting experts to shallow trees, MoDT offers localized explanations and transparent expert assignments, making it especially suitable for scientific and tabular-data applications where interpretability is a primary concern.

However, despite these advantages, MoDT inherits an important structural limitation from its design: all experts belong to the same hypothesis class. Consequently, every region of the feature space is modeled using decision trees, regardless of whether an alternative model family would provide a more natural representation. While decision trees are highly interpretable and flexible, their axis-aligned partition structure may be inefficient for representing smooth linear separators, elliptical decision boundaries, anisotropic Gaussian class-conditionals, or other geometric structures that arise naturally in many real-world datasets~\cite{breiman1984,piltaver2016,quinlan1993}.

From a statistical perspective, this restriction may be viewed as a form of local model mismatch. Although the gating mechanism partitions the input space adaptively, the set of admissible local models remains globally fixed. Consequently, adaptation occurs only through parameter variation within a single hypothesis class rather than through the selection of fundamentally different inductive biases. In practice, this often necessitates increasingly complex tree structures to approximate decision boundaries that could be represented more naturally by alternative model classes such as support vector machines~\cite{cortes1995} or quadratic discriminant models~\cite{hastie2009}.

This observation motivates the central idea of the present work. Many datasets exhibit heterogeneous local structure in which different regions are more naturally represented by different model families. For example, one region may be approximately linearly separable and therefore well modeled by a support vector machine, while another may exhibit covariance-driven quadratic boundaries better captured by quadratic discriminant analysis, and yet another may require rule-based axis-aligned partitions naturally represented by decision trees. The original MoE framework already provides a principled mechanism for routing observations toward the most appropriate expert through responsibility-based gating~\cite{jacobs1991,jordan1994}. Extending this framework to allow experts with heterogeneous inductive biases is therefore both natural and statistically appealing.

Building on the interpretability of MoDT and the probabilistic foundations of classical MoE models, we propose a heterogeneous mixture-of-experts framework that integrates decision trees, linear support vector machines, and quadratic discriminant analysis experts under a common linear softmax gating mechanism. To ensure coherent probabilistic inference, expert outputs are represented as conditional class probabilities and parameter estimation is performed through a generalized Expectation--Maximization procedure~\cite{dempster1977,bishop2006}. Unlike existing interpretable MoE approaches that rely on a single expert family, the proposed framework enables adaptive region-specific inductive bias selection while preserving transparency in both expert assignment and prediction.

The contributions of this work are threefold. First, we introduce an interpretable heterogeneous mixture-of-experts architecture that combines complementary expert families within a unified probabilistic framework. Second, we develop a generalized EM estimation procedure together with theoretical results establishing monotone ascent properties of the proposed gating updates. Third, through experiments on a collection of real-world and synthetic datasets, we demonstrate that the proposed framework achieves predictive performance comparable to strong baselines while providing meaningful expert specialization patterns that reveal how different inductive biases are selected across regions of the feature space.

\section{Problem Formulation and Discussion}
\label{sec:problem}

We consider a supervised multi-class classification problem in which we observe independent
and identically distributed samples
\[
  \mathcal{D} = \{(\boldsymbol{x}_i, y_i)\}_{i=1}^n
\]
drawn from an unknown joint distribution $P(\boldsymbol{X}, Y)$ defined on the product
space $\mathcal{X} \times \mathcal{Y}$. Here $\mathcal{X} \subset \mathbb{R}^{p+1}$ denotes
the feature space, where each $\boldsymbol{x}_i$ is augmented with an intercept term, and
$\mathcal{Y} = \{1,\dots,C\}$ denotes a finite label space.

Our objective is to construct an interpretable estimator of the conditional distribution
$P(Y \mid \boldsymbol{x})$ that generalizes to unseen data. We assume that the true
conditional mechanism exhibits \emph{region-specific structural heterogeneity}, in the sense
that no single hypothesis class provides a globally parsimonious representation of the Bayes
decision boundary.

Specifically, we posit the existence of a latent categorical variable $Z \in \{1,\dots,e\}$
such that the conditional distribution admits the representation
\begin{equation}
  P(Y = y \mid \boldsymbol{x})
  = \sum_{j=1}^{e} P(Z=j \mid \boldsymbol{x})\, P(Y=y \mid \boldsymbol{x}, Z=j).
  \label{eq:mixture_representation}
\end{equation}
This formulation induces a soft partition of the input space, allowing distinct conditional
mechanisms to dominate in different regions.

Classical MoE models~\cite{jacobs1991} operationalize~\eqref{eq:mixture_representation}
using parametric gating functions and homogeneous expert families. The MoDT framework
utilizes a linear gating function and constrains experts to interpretable tree-based
components, thereby enhancing transparency. However, restricting all experts to a single
hypothesis class imposes a common inductive bias across regions of the feature space. When
the underlying conditional structure deviates from axis-aligned partitions, this constraint
may induce regional model misspecification, forcing adaptation through increased model
complexity rather than appropriate bias selection.

The statistical problem addressed in this work is therefore the following: can one construct
a coherent and interpretable mixture estimator that permits heterogeneous expert families
while preserving likelihood-based inference and stable specialization behaviour?

To this end, we define a heterogeneous model space
\[
  \mathcal{F} = \bigcup_{j=1}^{e} \mathcal{F}_j,
\]
where each $\mathcal{F}_j$ denotes a distinct hypothesis class with complementary inductive
bias. In the present work, the expert families include: (i)~shallow decision trees for
rule-based axis-aligned structure, (ii)~linear support vector machines for
maximum-margin hyperplanar separation, and (iii)~quadratic discriminant models for
class-conditional covariance structure.

Let $\boldsymbol{\Theta}$ denote the parameters of a linear softmax gating mechanism
\[
  g_j(\boldsymbol{x};\boldsymbol{\Theta})
  = \frac{\exp(\boldsymbol{x}^\top \boldsymbol{\theta}_j)}
         {\sum_{k=1}^{e} \exp(\boldsymbol{x}^\top \boldsymbol{\theta}_k)},
\]
and let $\boldsymbol{\Psi} = (\psi_1,\dots,\psi_e)$ denote the collection of
expert-specific parameters. The observed-data likelihood under the heterogeneous mixture
model is
\begin{equation}
  \ell(\boldsymbol{\Theta},\boldsymbol{\Psi})
  = \sum_{i=1}^{n} \log\!\left(
      \sum_{j=1}^{e} g_j(\boldsymbol{x}_i;\boldsymbol{\Theta})\,
      p_j(y_i \mid \boldsymbol{x}_i; \psi_j)
    \right),
  \label{eq:observed_likelihood}
\end{equation}
where $p_j(y \mid \boldsymbol{x};\psi_j)$ denotes the conditional density associated with
expert~$j$.

For expert classes that do not inherently produce probabilistic outputs (e.g., margin-based
linear support vector machines), we employ calibrated mappings from decision scores to
conditional class probabilities (e.g., via \emph{Platt scaling})\cite{platt1999,niculescu2005,guo2017} to ensure that
\eqref{eq:observed_likelihood} defines a coherent likelihood function. This guarantees
compatibility with likelihood-based estimation consistent with modern probabilistic machine learning formulations \cite{murphy2022} and preserves probabilistic interpretability
of the mixture.

Parameter estimation proceeds via a generalized Expectation Maximization~\cite{bishop2006}
procedure that iteratively maximizes~\eqref{eq:observed_likelihood}. The resulting estimator
$(\hat{\boldsymbol{\Theta}}, \hat{\boldsymbol{\Psi}})$ is designed to achieve
region-adaptive inductive bias selection: in regions where a single hypothesis class is
sufficient, the gating mechanism may concentrate its mass on that expert, whereas in
structurally heterogeneous regimes, responsibility is distributed across complementary
experts.

The goal is thus to reconcile interpretability, coherence, and heterogeneous inductive bias
within a unified likelihood-based framework.

\section{Model Specification}
\label{sec:model}

Let $\boldsymbol{x}_1,\dots,\boldsymbol{x}_n \in \mathbb{R}^{p+1}$ denote the observed
covariate vectors, treated as fixed and non-random. Each vector is written as
$\boldsymbol{x}_i=(1,\boldsymbol{z}_i^\top)^\top$, where the first coordinate corresponds
to an intercept term and $\boldsymbol{z}_i \in \mathbb{R}^p$ denotes the non-intercept
covariates. Define the design matrix
\[
  \boldsymbol{X}
  = \begin{bmatrix} \boldsymbol{x}_1^\top \\ \vdots \\ \boldsymbol{x}_n^\top \end{bmatrix}
  \in \mathbb{R}^{n\times(p+1)}.
\]
Throughout this section we assume that
\[
  \operatorname{rank}(\boldsymbol{X}) = p+1.
\]
This condition is imposed to ensure injectivity of the normalized softmax gating
parametrization and to avoid degeneracy arising from linear dependence among covariates.

\subsection{Parameter Space}

We assume that the heterogeneous mixture model is indexed by
$(\boldsymbol\Theta,\boldsymbol\Psi)$, where
$\boldsymbol\Theta=(\boldsymbol\theta_1,\dots,\boldsymbol\theta_e)$ denotes the gating
parameters and $\boldsymbol\Psi=(\boldsymbol\psi_1,\dots,\boldsymbol\psi_e)$ denotes the
collection of expert-specific parameters.

The gating parameters satisfy
\[
  \boldsymbol\theta_j \in \mathbb{R}^{p+1}, \qquad j=1,\dots,e.
\]
As established below in Section~\ref{sec:estimation}, the softmax parametrization exhibits
an additive redundancy; after normalization, the free gating parameter space is identified
with $\mathbb{R}^{(p+1)(e-1)}$.

Further, it is assumed that each expert~$j$ belongs to a finite-dimensional parametric
family indexed by
\[
  \boldsymbol\psi_j \in \boldsymbol\Psi_j \subset \mathbb{R}^{d_j},
\]
where $d_j<\infty$ denotes the dimension of the $j$th expert model. This finite-dimensional
Euclidean representation reflects the parametric nature of the expert families. The
assumption $\boldsymbol\psi_j \in \mathbb{R}^{d_j}$ ensures that each expert family admits
a smooth coordinate representation suitable for likelihood-based analysis.

We assume that the overall parameter space
\[
  \boldsymbol\Omega
  = \bigl\{ (\boldsymbol\Theta,\boldsymbol\Psi) :
      \boldsymbol\theta_j \in \mathbb{R}^{p+1},\;
      \boldsymbol\psi_j \in \boldsymbol\Psi_j \bigr\}
\]
is compact. Combined with continuity of the log-likelihood, this guarantees existence of
maximizers of the observed-data likelihood, and that the target parameter lies in the
interior of $\boldsymbol\Omega$.

\section{Estimation via Expectation--Maximization}
\label{sec:estimation}

In Section~\ref{sec:problem}, the heterogeneous mixture model was specified through the
conditional density
\begin{equation}
  p(y_i \mid \boldsymbol x_i;\boldsymbol\Theta,\boldsymbol\Psi)
  = \sum_{j=1}^{e}
      g_j(\boldsymbol x_i;\boldsymbol\Theta)\,
      p_j(y_i \mid \boldsymbol x_i;\boldsymbol\psi_j),
\end{equation}
where $g_j(\boldsymbol x_i;\boldsymbol\Theta)$ denotes the softmax gating probability and
$p_j(y_i\mid x_i;\psi_j)$ denotes the conditional density associated with expert~$j$.

Direct maximization of the observed-data log-likelihood
\begin{equation}
  \ell(\boldsymbol\Theta,\boldsymbol\Psi)
  = \sum_{i=1}^{n}
      \log\!\left(
        \sum_{j=1}^{e}
          g_j(\boldsymbol x_i;\boldsymbol\Theta)\,
          p_j(y_i\mid \boldsymbol x_i;\boldsymbol\psi_j)
      \right)
\end{equation}
is analytically intractable due to the coupling between the gating and expert parameters
induced by the logarithm of the mixture sum.

To facilitate optimization, we introduce latent allocation variables
\[
  Z_i \in \{1,\ldots,e\},
\]
where $Z_i=j$ indicates that observation~$i$ is generated by expert~$j$. This
latent-variable representation naturally leads to an Expectation--Maximization (EM)
framework \cite{dempster1977,mclachlan2008}.

The proposed estimation procedure combines:
\begin{enumerate}
  \item a fully probabilistic E-step derived from the heterogeneous mixture likelihood, and
  \item a computationally efficient regression-based generalized M-step inspired by the
        Mixture of Decision Trees framework.
\end{enumerate}
The resulting algorithm therefore constitutes a generalized Expectation--Maximization (GEM)
procedure.

\subsection{Complete-Data Representation}

If the latent variables $Z_1,\ldots,Z_n$ were observed, the complete-data likelihood would
factorize as
\[
  L_c(\boldsymbol\Theta,\boldsymbol\Psi)
  = \prod_{i=1}^{n} \prod_{j=1}^{e}
      \Bigl[
        g_j(\boldsymbol x_i;\boldsymbol\Theta)\,
        p_j(y_i\mid \boldsymbol x_i;\boldsymbol\psi_j)
      \Bigr]^{\mathbbm{1}(Z_i=j)},
\]
where $\mathbbm{1}(\cdot)$ denotes the indicator function.

Taking logarithms yields the complete-data log-likelihood:
\begin{equation}
  \ell_c(\boldsymbol\Theta,\boldsymbol\Psi)
  = \sum_{i=1}^{n} \sum_{j=1}^{e}
      \mathbbm{1}(Z_i=j)
      \Bigl[
        \log g_j(\boldsymbol x_i;\boldsymbol\Theta)
        + \log p_j(y_i\mid \boldsymbol x_i;\boldsymbol\psi_j)
      \Bigr].
\end{equation}

Because the latent allocations are unobserved, EM proceeds by iteratively replacing the
indicator variables by their posterior conditional expectations.

\subsection{E-Step}

At iteration~$t$, let $(\boldsymbol\Theta^{(t)},\boldsymbol\Psi^{(t)})$ denote the current
parameter estimates. The E-step computes the posterior distribution of the latent allocation
variables conditioned on the observed data:
\[
  \tau_{ij}^{(t)}
  = P(Z_i=j \mid y_i,\boldsymbol x_i;\boldsymbol\Theta^{(t)},\boldsymbol\Psi^{(t)}),
\]
where $i$ ranges from $1$ to $n$ and $j$ varies from $1$ to $e$.

Applying Bayes' theorem yields
\begin{equation}
  \tau_{ij}^{(t)}
  = \frac{
      g_j(\boldsymbol x_i;\boldsymbol\Theta^{(t)})\,
      p_j(y_i\mid \boldsymbol x_i;\boldsymbol\psi_j^{(t)})
    }{
      \sum_{k=1}^{e}
        g_k(\boldsymbol x_i;\boldsymbol\Theta^{(t)})\,
        p_k(y_i\mid \boldsymbol x_i;\boldsymbol\psi_k^{(t)})
    }.
\end{equation}

The quantities $\tau_{ij}^{(t)}$ are referred to as \emph{posterior responsibilities} and
satisfy
\[
  0 \leq \tau_{ij}^{(t)} \leq 1, \qquad \sum_{j=1}^{e}\tau_{ij}^{(t)}=1.
\]

Unlike heuristic confidence-based responsibility assignments, these posterior probabilities
arise directly from the latent-variable mixture likelihood and therefore possess a rigorous
probabilistic interpretation.

Define the posterior responsibility matrix
\[
  \boldsymbol T^{(t)}
  = \bigl(\tau_{ij}^{(t)}\bigr)_{i,j} \in \mathbb{R}^{n\times e}.
\]

The auxiliary EM functional is then given by
\[
  Q\!\bigl(\boldsymbol\Theta,\boldsymbol\Psi
    \mid \boldsymbol\Theta^{(t)},\boldsymbol\Psi^{(t)}\bigr)
  = \mathbb{E}\!\left[
      \ell_c(\boldsymbol\Theta,\boldsymbol\Psi)
      \mid \mathcal{D};\,
      \boldsymbol\Theta^{(t)},\boldsymbol\Psi^{(t)}
    \right].
\]

Substituting the complete-data log-likelihood yields
\begin{equation}
  Q(\boldsymbol\Theta,\boldsymbol\Psi)
  = \sum_{i=1}^{n} \sum_{j=1}^{e}
      \tau_{ij}^{(t)}
      \Bigl[
        \log g_j(\boldsymbol x_i;\boldsymbol\Theta)
        + \log p_j(y_i\mid \boldsymbol x_i;\boldsymbol\psi_j)
      \Bigr].
\end{equation}

The auxiliary function separates additively into gating-specific and expert-specific
components:
\[
  Q(\boldsymbol\Theta,\boldsymbol\Psi)
  = Q_{\mathrm{gate}}(\boldsymbol\Theta)
    + \sum_{j=1}^{e} Q_j(\boldsymbol\psi_j),
\]
where
\begin{equation}
  Q_{\mathrm{gate}}(\boldsymbol\Theta)
  = \sum_{i=1}^{n} \sum_{j=1}^{e}
      \tau_{ij}^{(t)} \log g_j(\boldsymbol x_i;\boldsymbol\Theta),
\end{equation}
and
\begin{equation}
  Q_j(\boldsymbol\psi_j)
  = \sum_{i=1}^{n}
      \tau_{ij}^{(t)} \log p_j(y_i\mid \boldsymbol x_i;\boldsymbol\psi_j).
\end{equation}

This separability permits independent optimization of the experts and the gating mechanism.

Since exact maximization of the gating likelihood requires iterative multinomial logistic
optimization, we instead adopt the regression-based surrogate update proposed in
MoDT~\cite{bru}. Consequently, the resulting procedure is a generalized EM (GEM) algorithm
rather than an exact EM algorithm.

\subsection{M-Step}

In the M-step, the auxiliary function is optimized with respect to the parameters
$(\boldsymbol\Theta,\boldsymbol\Psi)$.

\subsubsection{Expert Updates}

For each expert~$j$, the optimization problem reduces to
\begin{equation}
  \widehat{\boldsymbol\psi}_j^{(t+1)}
  = \arg\max_{\boldsymbol\psi_j\in\boldsymbol\Psi_j}
      \sum_{i=1}^{n}
        \tau_{ij}^{(t)} \log p_j(y_i\mid \boldsymbol x_i;\boldsymbol\psi_j).
\end{equation}

Thus, each expert update corresponds to a weighted maximum likelihood estimation problem
with observation weights given by the posterior responsibilities.

Accordingly:
\begin{itemize}
  \item decision tree experts are trained using weighted impurity minimization,
  \item support vector machine experts are trained using weighted margin optimization
        together with probability calibration,
  \item quadratic discriminant experts are estimated using weighted class means and
        covariance matrices.
\end{itemize}

Since the expert objectives are mutually independent, all experts may be optimized
separately and in parallel.

\subsubsection{Regression-Based Gating Update}

The gating objective is
\begin{equation}
  Q_{\mathrm{gate}}(\boldsymbol\Theta)
  = \sum_{i=1}^{n} \sum_{j=1}^{e}
      \tau_{ij}^{(t)} \log g_j(x_i;\boldsymbol\Theta).
\end{equation}

Exact maximization of this objective corresponds to multinomial logistic regression with
soft labels $\tau_{ij}^{(t)}$. Although the resulting optimization problem is convex, it
does not admit a closed-form solution and typically requires iterative numerical procedures
such as IRLS or quasi-Newton optimization at every EM iteration.

To avoid this repeated high-dimensional optimization, we follow the computational strategy
adopted in the MoDT~\cite{bru} framework and replace the exact multinomial likelihood
maximization by a least-squares surrogate update. Specifically, rather than directly
optimizing the softmax log-likelihood, we model the discrepancy between the posterior
responsibilities and the current gating assignments through a linear regression
approximation.

This surrogate formulation yields a computationally efficient closed-form update while
preserving the directional alignment between the posterior allocation structure obtained in
the E-step and the gating probabilities induced by the current parameter estimates.

Define the current gating probability matrix
\[
  \boldsymbol G^{(t)}
  = \bigl(g_j(\boldsymbol x_i;\boldsymbol\Theta^{(t)})\bigr)_{i,j}
  \in \mathbb{R}^{n\times e}.
\]

We define the residual responsibility matrix
\begin{equation}
  \boldsymbol R^{(t)} = \boldsymbol T^{(t)} - \boldsymbol G^{(t)}.
\end{equation}

The matrix $\boldsymbol{R}^{(t)}$ measures the discrepancy between the posterior expert
allocations implied by the complete-data likelihood, and the current gating allocations
induced by the softmax gating network. Positive entries $R_{ij}^{(t)} > 0$ indicate
under-allocation of observation~$i$ to expert~$j$, while negative entries indicate
over-allocation.

Let
\[
  \boldsymbol X
  = \begin{bmatrix} \boldsymbol x_1^\top \\ \vdots \\ \boldsymbol x_n^\top \end{bmatrix}
  \in \mathbb{R}^{n\times (p+1)}
\]
denote the design matrix. We seek a linear perturbation of the gating parameters that
approximately aligns the gating probabilities with the posterior responsibilities. As in
MoDT~\cite{bru}, this is achieved by solving the multivariate least-squares problem
\begin{equation}
  \widehat{\boldsymbol B}^{(t)}
  = \arg\min_{\boldsymbol B\in\mathbb{R}^{(p+1)\times e}}
      \left\|
        \boldsymbol R^{(t)} - \boldsymbol X \boldsymbol B
      \right\|_{\mathcal{F}}^2,
\end{equation}
where $\|\cdot\|_{\mathcal{F}}$ denotes the Frobenius norm. Equivalently,
\[
  \widehat{\boldsymbol B}^{(t)}
  = \arg\min_B
      \sum_{i=1}^{n} \sum_{j=1}^{e}
        \left(
          \tau_{ij}^{(t)}
          - g_j(\boldsymbol x_i;\boldsymbol\Theta^{(t)})
          - x_i^\top \boldsymbol\beta_j
        \right)^2.
\]

Assuming $\boldsymbol X^\top \boldsymbol X$ is nonsingular, the minimizer admits the
closed-form expression
\begin{equation}
  \widehat{\boldsymbol B}^{(t)}
  = (\boldsymbol X^\top \boldsymbol X)^{-1} \boldsymbol X^\top \boldsymbol R^{(t)}.
\end{equation}

The gating parameters are then updated according to
\begin{equation}
  \boldsymbol \Theta^{(t+1)}
  = \boldsymbol \Theta^{(t)} + \gamma_t \widehat{\boldsymbol B}^{(t)},
\end{equation}
where $\gamma_t>0$ denotes a learning-rate parameter. This update may be interpreted as a
first-order correction step that moves the gating probabilities toward the posterior
allocation structure computed in the E-step.

\subsection{Local Generalized EM Interpretation}

The proposed gating update does not maximize $Q_{\mathrm{gate}}(\boldsymbol{\Theta})$
exactly. Instead, it performs an approximate ascent step based on the discrepancy between
the posterior responsibilities obtained from the E-step and the current gating
probabilities. Consequently, the resulting optimization procedure constitutes a generalized
Expectation--Maximization (GEM) algorithm rather than an exact EM algorithm.

The following theorem establishes that the proposed least-squares gating update defines a
monotone ascent step for the gating auxiliary objective under suitable learning-rate
conditions.

\begin{theorem}
\label{thm:ascent}
Suppose that:
\begin{enumerate}
  \item the design matrix satisfies $\mathrm{rank}(\boldsymbol{X})=p+1$;
  \item the gating auxiliary function
        \[
          Q_{\mathrm{gate}}(\boldsymbol{\Theta})
          = \sum_{i=1}^{n} \sum_{j=1}^{e}
              \tau_{ij}^{(t)}
              \log g_j(\boldsymbol{x}_i;\boldsymbol{\Theta})
        \]
        is twice continuously differentiable.
\end{enumerate}
Define $\boldsymbol{R}^{(t)} = \boldsymbol{T}^{(t)} - \boldsymbol{G}^{(t)}$ and let
\[
  \widehat{\boldsymbol{B}}^{(t)}
  = (\boldsymbol{X}^\top\boldsymbol{X})^{-1} \boldsymbol{X}^\top \boldsymbol{R}^{(t)}.
\]
Further define
\[
  A_t
  = \bigl\langle
      \nabla_{\boldsymbol{\Theta}} Q_{\mathrm{gate}}(\boldsymbol{\Theta}^{(t)}),\,
      \widehat{\boldsymbol{B}}^{(t)}
    \bigr\rangle
\]
and
\[
  C_t
  = \bigl\langle
      \widehat{\boldsymbol{B}}^{(t)},\,
      \nabla^2_{\boldsymbol{\Theta}} Q_{\mathrm{gate}}(\widetilde{\boldsymbol{\Theta}}^{(t)})\,
      \widehat{\boldsymbol{B}}^{(t)}
    \bigr\rangle,
\]
where $\widetilde{\boldsymbol{\Theta}}^{(t)}
= \boldsymbol{\Theta}^{(t)} + \eta_t\gamma_t\widehat{\boldsymbol{B}}^{(t)}$ for some
$\eta_t\in(0,1)$.

Then:
\begin{enumerate}
  \item $A_t \geq 0$, with strict inequality whenever
        $\boldsymbol{R}^{(t)} \notin \mathrm{Null}(\boldsymbol{P}_{\boldsymbol X})$, where
        $\boldsymbol P_{\boldsymbol X}
         = \boldsymbol X(\boldsymbol X^\top\boldsymbol X)^{-1}\boldsymbol X^\top$;
  \item the update
        $\boldsymbol{\Theta}^{(t+1)}
         = \boldsymbol{\Theta}^{(t)} + \gamma_t\widehat{\boldsymbol B}^{(t)}$
        satisfies
        $Q_{\mathrm{gate}}(\boldsymbol{\Theta}^{(t+1)})
         > Q_{\mathrm{gate}}(\boldsymbol{\Theta}^{(t)})$
        whenever $0 < \gamma_t < 2A_t / |C_t|$.
\end{enumerate}
\end{theorem}

\begin{proof}
The softmax gating probabilities are given by
\[
  g_j(\boldsymbol{x}_i;\boldsymbol{\Theta})
  = \frac{\exp(\boldsymbol{x}_i^\top\boldsymbol{\theta}_j)}
         {\sum_{k=1}^{e}\exp(\boldsymbol{x}_i^\top\boldsymbol{\theta}_k)}.
\]
Therefore,
\[
  \log g_j(\boldsymbol{x}_i;\boldsymbol{\Theta})
  = \boldsymbol{x}_i^\top\boldsymbol{\theta}_j
    - \log\!\left(\sum_{k=1}^{e}\exp(\boldsymbol{x}_i^\top\boldsymbol{\theta}_k)\right).
\]

Substituting into $Q_{\mathrm{gate}}$ and using the identity
$\sum_{j=1}^{e}\tau_{ij}^{(t)}=1$ yields
\[
  Q_{\mathrm{gate}}(\boldsymbol{\Theta})
  = \sum_{i=1}^{n} \sum_{j=1}^{e}
      \tau_{ij}^{(t)} \boldsymbol{x}_i^\top \boldsymbol{\theta}_j
    - \sum_{i=1}^{n}
      \log\!\left(\sum_{k=1}^{e}\exp(\boldsymbol{x}_i^\top\boldsymbol{\theta}_k)\right).
\]

Differentiating with respect to $\boldsymbol{\theta}_j$ gives
\[
  \nabla_{\boldsymbol{\theta}_j} Q_{\mathrm{gate}}(\boldsymbol{\Theta})
  = \sum_{i=1}^{n}
      \bigl(\tau_{ij}^{(t)} - g_j(\boldsymbol{x}_i;\boldsymbol{\Theta})\bigr)
      \boldsymbol{x}_i.
\]

Stacking the gradients columnwise gives
\[
  \nabla_{\boldsymbol{\Theta}} Q_{\mathrm{gate}}(\boldsymbol{\Theta})
  = \boldsymbol X^\top \boldsymbol R^{(t)}.
\]

Using the definition of $\widehat{\boldsymbol B}^{(t)}$,
\[
  A_t
  = \mathrm{tr}\!\left[
      (\boldsymbol R^{(t)})^\top
      \boldsymbol P_{\boldsymbol X}
      \boldsymbol R^{(t)}
    \right].
\]

Since $\boldsymbol P_{\boldsymbol X}$ is symmetric positive semidefinite, $A_t \geq 0$,
with equality if and only if $\boldsymbol R^{(t)} \in \mathrm{Null}(\boldsymbol P_{\boldsymbol X})$.
This proves Part~1.

For the Hessian, differentiating the softmax probabilities gives
\[
  \frac{\partial g_{ij}}{\partial\boldsymbol{\theta}_\ell}
  = g_{ij}(\delta_{j\ell}-g_{i\ell})\boldsymbol x_i,
\]
where $\delta_{j\ell}$ denotes the Kronecker delta. Define
$\boldsymbol g_i = (g_{i1},\dots,g_{ie})^\top$ and
$\boldsymbol\Sigma_i = \operatorname{Diag}(\boldsymbol g_i) - \boldsymbol g_i\boldsymbol g_i^\top$.
The Hessian may be written compactly as
\[
  \nabla^2 Q_{\mathrm{gate}}
  = -\sum_{i=1}^{n}
      \bigl[\boldsymbol\Sigma_i \otimes (\boldsymbol x_i\boldsymbol x_i^\top)\bigr].
\]

This confirms $\nabla^2 Q_{\mathrm{gate}}$ is negative semi-definite, hence $C_t \leq 0$.

By the multivariate second-order Taylor expansion,
\[
  Q_{\mathrm{gate}}(\boldsymbol\Theta^{(t+1)})
  - Q_{\mathrm{gate}}(\boldsymbol\Theta^{(t)})
  = \gamma_t A_t + \tfrac{\gamma_t^2}{2} C_t.
\]

Since $C_t \leq 0$,
\[
  Q_{\mathrm{gate}}(\boldsymbol\Theta^{(t+1)})
  - Q_{\mathrm{gate}}(\boldsymbol\Theta^{(t)})
  \geq \gamma_t\!\left(A_t - \tfrac{\gamma_t}{2}|C_t|\right) > 0
\]
whenever $0 < \gamma_t < 2A_t/|C_t|$. This proves Part~2.
\end{proof}

\begin{theorem}
\label{thm:learning_rate}
Under the multinomial softmax gating parametrization,
\[
  \frac{2A_t}{|C_t|} \geq 2.
\]
Consequently, every learning-rate sequence satisfying $0 < \gamma_t < 2$ guarantees
monotone ascent of the gating auxiliary function:
\[
  Q_{\mathrm{gate}}(\boldsymbol{\Theta}^{(t+1)})
  \geq Q_{\mathrm{gate}}(\boldsymbol{\Theta}^{(t)})
\]
for all~$t$. Hence, the resulting optimization procedure defines a valid generalized EM
algorithm \cite{wu1983}.
\end{theorem}

\begin{proof}
From Theorem~\ref{thm:ascent},
\[
  A_t
  = \mathrm{tr}\!\left[(\boldsymbol{R}^{(t)})^{\top}\boldsymbol{P}_{\boldsymbol X}\boldsymbol R^{(t)}\right]
  = \|\boldsymbol P_{\boldsymbol X}\boldsymbol R^{(t)}\|_{\mathcal{F}}^2
  = \|\boldsymbol X\widehat{\boldsymbol B}^{(t)}\|_\mathcal{F}^2.
\]

To bound $|C_t|$, we first show $\boldsymbol\Sigma_i \preceq \boldsymbol I$. For any
$\boldsymbol z \in \mathbb{R}^e$,
\[
  \boldsymbol z^\top\boldsymbol\Sigma_i\boldsymbol z
  = \sum_{j=1}^{e} g_{ij}z_j^2
    - \!\left(\sum_{j=1}^{e} g_{ij}z_j\right)^2
  \leq \sum_{j=1}^{e} g_{ij}z_j^2
  \leq \sum_{j=1}^{e} z_j^2
  = \|\boldsymbol z\|^2,
\]
since $0 \leq g_{ij} \leq 1$ for all $i,j$. Hence $\boldsymbol\Sigma_i \preceq \boldsymbol I$.

Setting $\boldsymbol v_i = \boldsymbol x_i^\top\widehat{\boldsymbol B}^{(t)}$ and using
the Kronecker representation of the Hessian gives
\[
  |C_t|
  = \sum_{i=1}^{n}\boldsymbol v_i^\top\boldsymbol\Sigma_i\boldsymbol v_i
  \leq \sum_{i=1}^{n}\|\boldsymbol v_i\|^2
  = \|\boldsymbol X\widehat{\boldsymbol B}^{(t)}\|_\mathcal{F}^2
  = A_t.
\]

Therefore $2A_t / |C_t| \geq 2$, and every $\gamma_t \in (0,2)$ satisfies the condition
$0 < \gamma_t < 2A_t/|C_t|$ from Theorem~\ref{thm:ascent}.
\end{proof}

The principal computational advantage of this formulation is that the gating update reduces
to a simple closed-form multivariate regression problem, thereby avoiding repeated iterative
optimization of the multinomial logistic objective. At the same time, the update remains
directly driven by the discrepancy between the posterior responsibilities and the current
gating probabilities, ensuring that the probabilistic structure learned during the E-step
continues to guide the parameter updates. Moreover, Theorem~\ref{thm:ascent} shows that the
resulting least-squares correction acts as a local ascent step for the gating auxiliary
function for sufficiently small learning rates, thereby providing the proposed procedure
with a generalized EM interpretation~\cite{dempster1977}.

The complete training procedure of the proposed heterogeneous mixture-of-experts framework is summarized in Algorithm~\ref{alg:training}. The algorithm alternates between the computation of posterior responsibilities in the E-step and the estimation of the expert and gating parameters in the M-step until convergence. The regression-based gating update described in Section~\ref{sec:estimation} is incorporated directly into the iterative optimization procedure.
\begin{algorithm}[htbp]
\caption{Training Procedure of the Proposed Heterogeneous Mixture-of-Experts}
\label{alg:training}
\footnotesize
\begin{algorithmic}[1]

\Require
Training data
$\mathcal D=\{(x_i,y_i)\}_{i=1}^{n}$,
expert families
$\mathcal F=\{\mathrm{DT},\mathrm{SVM},\mathrm{QDA}\}$,
number of experts $e$,
maximum iterations $T$,
initial learning rate $\gamma_0$.

\Ensure
Estimated gating parameters
$\widehat{\Theta}$,
expert parameters
$\widehat{\Psi}$.

\State Initialize gating parameters $\Theta^{(0)}$.

\State Compute initial gating probabilities
$G^{(0)}$.

\For{$t=0,\ldots,T-1$}

\State \textbf{E-step}

\For{each observation $i$}

\For{each expert $j$}

\State Compute posterior responsibility

\[
\tau_{ij}^{(t)}
=
\frac{
g_j(x_i;\Theta^{(t)})
p_j(y_i|x_i;\psi_j^{(t)})
}
{
\sum_{k=1}^{e}
g_k(x_i;\Theta^{(t)})
p_k(y_i|x_i;\psi_k^{(t)})
}.
\]

\EndFor
\EndFor

\State Form the responsibility matrix
$T^{(t)}$.

\State

\textbf{M-step}

\For{each expert $j$}

\If{expert $j$ is DT}

\State Train a weighted Decision Tree.

\ElsIf{expert $j$ is Linear SVM}

\State Train a weighted Linear SVM.

\State Calibrate probabilities using Platt scaling.

\ElsIf{expert $j$ is QDA}

\State Estimate weighted means and covariance matrices.

\EndIf

\EndFor

\State Compute residual matrix $R^{(t)}
=
T^{(t)}
-
G^{(t)}.$

\State Compute $B^{(t)}
=
(X^\top X)^{-1}
X^\top
R^{(t)}.$

\State Update gating parameters $\Theta^{(t+1)}
=
\Theta^{(t)}
+
\gamma_t B^{(t)}.$
\State Update learning rate
$\gamma_{t+1}
=
\eta\gamma_t.$

\If{log-likelihood increment $<\varepsilon$}

\State Break.

\EndIf

\EndFor

\State Return
$\widehat{\Theta}$,
$\widehat{\Psi}$.

\end{algorithmic}
\end{algorithm}
\normalsize
\section{Application}
\label{sec:application}

In this section the overall classification accuracy performance of the heterogeneous mixture
model is evaluated and compared with that of the MoDT architecture. The
heterogeneous mixture model is further contrasted with non-interpretable ML techniques such
as Random Forest.

Experiments are conducted on several publicly available benchmark datasets to evaluate the
proposed heterogeneous mixture framework. Categorical features are one-hot encoded if
necessary to make them compatible with the underlying expert models. Unless otherwise stated, all experiments employ the training procedure described in Algorithm~\ref{alg:training}. The same optimization settings and convergence criteria are used across all datasets to ensure a fair comparison between the proposed heterogeneous framework and the competing methods. As the Python library
\textsc{scikit-learn}~\cite{scikit-learn} is used for the decision trees within the
implementation of the heterogeneous mixture model, the decision trees used for MoDT are
also implemented using \textsc{scikit-learn}. Furthermore, the proposed framework is
compared against Random Forests~\cite{breiman2001rf}, which are widely recognized for their strong predictive
performance. Although Random Forests are frequently regarded as strong predictive baselines, they typically provide considerably less transparency than interpretable mixture-based models. Consequently, achieving predictive performance comparable to Random Forests may itself be viewed as a favorable outcome in settings where interpretability is an important modeling objective. To allow a fair comparison,
the same preprocessing steps included in the heterogeneous model are also applied to the
training data of all alternative models.

For each dataset, the data is split into training data (75\%) and test data (25\%) using
randomized train-test splits. The splitting procedure is repeated 10 times with different
random seeds in order to obtain statistically stable performance estimates, and the mean
classification accuracy and standard deviation over all repetitions are reported as the
final performance.

The proposed heterogeneous model uses a fixed optimization configuration throughout all
experiments. In particular, the initialization learning rate, the learning rate decay and
the initialization procedure are fixed across all datasets and repetitions. We fix the
initialization procedure to $\text{Random\_init}(40)$ to minimize the variability of the
initialization due to stochastic effects and to ensure reproducibility. As a result,
experimental variability is introduced mainly by random train-test splits, rather than
repeated hyperparameter tuning.

The motivation for using a fixed hyperparameter setting is two-fold. First, the main goal
of this work is to assess the impact of heterogeneous expert composition in the MoDT
framework, rather than optimizing dataset-specific hyperparameters. By fixing optimization
settings, improvements in predictive performance can be more directly attributed to the
heterogeneous expert architecture itself rather than to aggressive hyperparameter search
procedures. Second, large-scale hyperparameter optimization can give over-optimistic
estimates, and may mask the intrinsic behavior of heterogeneous expert specialization. Thus,
the proposed evaluation protocol focuses on reproducibility, architectural fairness and
robustness across multiple datasets under a common optimization setting.

In addition, the proposed approach allows for analysis beyond predictive performance.
Because the heterogeneous mixture framework integrates decision trees, support vector machines
and quadratic discriminant analysis experts in a common gating mechanism, expert utilization
statistics are additionally analyzed to investigate specialization behavior across datasets
and feature distributions.

We also conduct an expert specialization analysis for the proposed heterogeneous mixture
framework. Specifically, we say that an expert dominates the heterogeneous gating mechanism
for the corresponding dataset if the average utilization of a single expert over multiple
experimental runs exceeds 90\%. In these cases, the dominant expert model is also trained
independently on the same dataset and the overall maximum predictive accuracy is reported.
This analysis offers additional insights into whether the heterogeneous framework benefits
primarily from collaborative expert interaction or whether certain datasets are inclined to
a particular expert type. For completeness, Algorithm~\ref{alg:prediction} summarizes the prediction procedure together with the expert specialization analysis used throughout the experimental study. Besides generating class predictions, the algorithm records the average utilization of each expert and identifies datasets for which a single expert dominates the gating mechanism.

\begin{algorithm}[t]
\caption{Prediction and Expert Specialization Analysis}
\label{alg:prediction}
\begin{algorithmic}[1]

\Require
Test data,
trained gating parameters
$\widehat{\Theta}$,
trained experts
$\widehat{\Psi}$.

\Ensure
Predicted labels and expert utilization statistics.

\For{each test observation $x$}

\For{each expert $j$}

\State Compute gating probability

$g_j(x;\widehat{\Theta}).$

\State Compute calibrated expert probability

$p_j(y|x;\widehat{\psi}_j).$

\EndFor

\State Compute mixture probability

$P(y|x)
=
\sum_{j=1}^{e}
g_j(x)
p_j(y|x).$

\State Predict
$\widehat y
=
\arg\max_y
P(y|x).$
\State Assign the observation to the expert with maximum gating probability.

\EndFor

\State Compute the average utilization of each expert.

\If{average utilization of an expert exceeds $90\%$}

\State Train the dominant expert independently on the same dataset.

\State Evaluate its standalone predictive accuracy.

\State Report maximum of standalone and heterogeneous model accuracies.

\Else

\State Report only the heterogeneous model accuracy.

\EndIf

\end{algorithmic}
\end{algorithm}

\subsection{Application on Synthetic Datasets}

\begin{table*}[!t]
  \centering
  \caption{Description of synthetic datasets.}
  \label{tab:synthetic_datasets}
  \resizebox{\textwidth}{!}{%
    \begin{tabular}{lccp{8.5cm}}
      \toprule
      \textbf{Dataset} & \textbf{Data Points} & \textbf{Classes} & \textbf{Description} \\
      \midrule
      overlapping\_rings             & 3{,}000       & 2 & Noisy concentric rings with overlap. \\
      rotated\_gaussian\_close       & 2{,}000       & 2 & Rotated Gaussian classes with overlapping quadratic boundaries. \\
      piecewise\_linear\_kink        & 3{,}000       & 2 & Piecewise linear decision boundary with a kink. \\
      anisotropic\_gaussian          & 600           & 2 & Two centered Gaussian classes with orthogonal anisotropic covariance structures. \\
      blob\_vs\_ring                 & 500  & 2 & Central Gaussian blob versus outer ring structure. \\
      multiclass\_anisotropic\_gaussians & 2{,}400   & 4 & Anisotropic Gaussian classes with different covariance orientations. \\
      hierarchical\_gaussians        & 3{,}000       & 6 & Hierarchical Gaussian subclasses forming two superclusters. \\
      radial\_sector\_multiclass     & 3{,}000       & 6 & Angular sector-based classes with nonlinear radial boundaries. \\
      \bottomrule
    \end{tabular}%
  }
\end{table*}

\begin{table*}[!t]
  \centering
  \caption{Comparison of models across synthetic datasets (mean accuracy $\pm$ std).}
  \label{tab:model_comparison}
  \resizebox{\textwidth}{!}{%
    \begin{tabular}{lcccccc}
      \toprule
      \multirow{2}{*}{Dataset}
        & \multirow{2}{*}{Random Forest}
        & \multicolumn{2}{c}{MoDT}
        & \multicolumn{2}{c}{Hetero-Mix Expert}
        & \multirow{2}{*}{Avg.\ Expert Usage (\%)} \\
      \cmidrule(lr){3-4}\cmidrule(lr){5-6}
        & & Depth 2 & Depth 3 & Depth 2 & Depth 3 & \\
      \midrule
      Overlapping Rings             & $0.77\pm0.01$ & $0.78\pm0.01$ & $0.78\pm0.01$ & $0.77\pm0.01$ & $\mathbf{0.79\pm0.02}$ & DT 75.4, QDA 24.6 \\
      Rotated Gaussian Slabs        & $0.69\pm0.01$ & $0.72\pm0.01$ & $0.72\pm0.01$ & $0.72\pm0.01$ & $\mathbf{0.73\pm0.01}$ & DT 62.5, SVM 20.7, QDA 16.8 \\
      Piecewise Linear Kink         & $\mathbf{0.95\pm0.01}$ & $0.92\pm0.01$ & $0.94\pm0.01$ & $0.91\pm0.01$ & $0.94\pm0.01$ & DT 18.5, SVM 39.6, QDA 41.9 \\
      Anisotropic Gaussian          & $0.88\pm0.02$ & $0.86\pm0.02$ & $0.85\pm0.02$ & $\mathbf{0.89\pm0.02}$ & $0.85\pm0.02$ & QDA 100 \\
      Inner Blob + Outer Ring       & $\mathbf{0.81\pm0.02}$ & $\mathbf{0.81\pm0.02}$ & $0.80\pm0.02$ & $\mathbf{0.81\pm0.03}$ & $0.80\pm0.02$ & DT 44.1, QDA 55.9 \\
      Multiclass Anisotropic Gaussian & $0.64\pm0.02$ & $0.63\pm0.04$ & $0.64\pm0.02$ & $\mathbf{0.69\pm0.02}$ & $\mathbf{0.69\pm0.02}$ & QDA 100 \\
      Hierarchical Gaussians        & $\mathbf{0.99\pm0.01}$ & $0.98\pm0.01$ & $0.98\pm0.01$ & $\mathbf{0.99\pm0.01}$ & $0.97\pm0.02$ & QDA 100 or DT 29.6, QDA 70.4 \\
      Radial Sector Multiclass      & $0.90\pm0.01$ & $0.89\pm0.02$ & $0.89\pm0.01$ & $\mathbf{0.92\pm0.01}$ & $\mathbf{0.92\pm0.01}$ & QDA 100 \\
      \bottomrule
    \end{tabular}%
  }
\end{table*}

\begin{figure*}
  \centering
  \includegraphics[width=\textwidth]{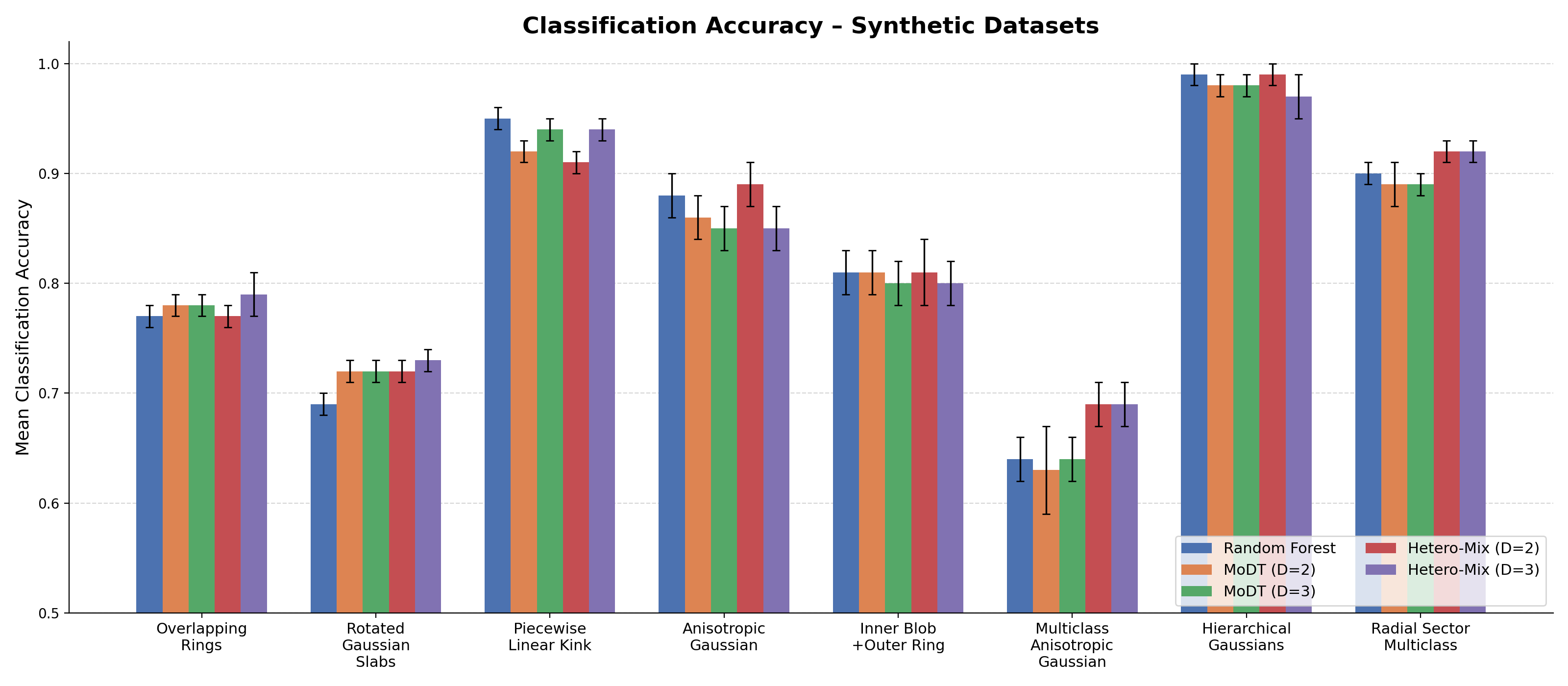}
  \caption{Mean classification accuracy on synthetic datasets. Error bars indicate one
           standard deviation over ten random train-test splits. Results are shown for
           Random Forest, MoDT (depth~2 and~3), and the proposed Hetero-Mix (depth~2
           and~3).}
  \label{fig:synthetic_accuracy}
\end{figure*}

\begin{figure*}
  \centering
  \includegraphics[width=\textwidth]{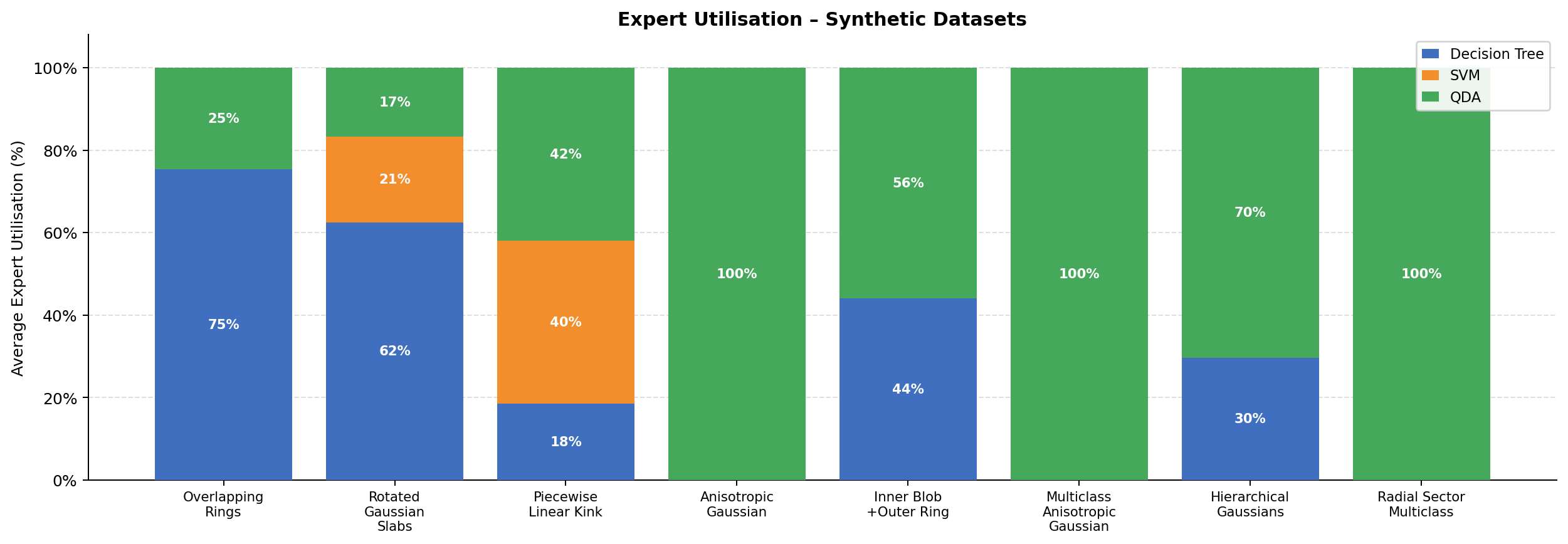}
  \caption{Average expert utilisation (\%) on synthetic datasets for the proposed
           heterogeneous mixture model. Each stacked bar shows the proportion of samples
           assigned to the Decision Tree, SVM, and QDA experts by the gating mechanism,
           averaged over ten train-test splits.}
  \label{fig:synthetic_experts}
\end{figure*}

Table~\ref{tab:synthetic_datasets} gives the descriptions of the generated datasets while
Table~\ref{tab:model_comparison} and Figure~\ref{fig:synthetic_accuracy} present the
experimental results on a diverse collection of synthetic benchmark datasets designed to
capture various geometric decision boundaries and covariance structures.
Figure~\ref{fig:synthetic_experts} illustrates the expert utilisation behaviour across
these datasets. The proposed heterogeneous mixture framework performs well on a
number of nonlinear and covariance-sensitive datasets and at the same time provides insight
into the behavior of expert specialization. The dominance of the QDA expert in the gating
mechanism for datasets like Anisotropic Gaussian, Hierarchical Gaussian and Radial Sector
Multiclass (Figure~\ref{fig:synthetic_experts}) indicates the suitability of
covariance-based quadratic decision boundaries for these distributions. Furthermore, for Piecewise Linear Kink and Rotated Gaussian Slabs the heterogeneous
framework reveals a significant collaboration between decision tree, SVM and QDA experts,
which indicates that different experts specialize in different regions of the feature space.
These observations demonstrate that the proposed heterogeneous gating framework is capable
of adaptively selecting suitable expert models according to intrinsic geometric
characteristics of the data distribution while achieving predictive performance competitive
with strong ensemble baselines such as Random Forests.

\begin{figure*}
\centering


\begin{subfigure}[!t]{0.43\textwidth}
\centering
\includegraphics[width=\textwidth]{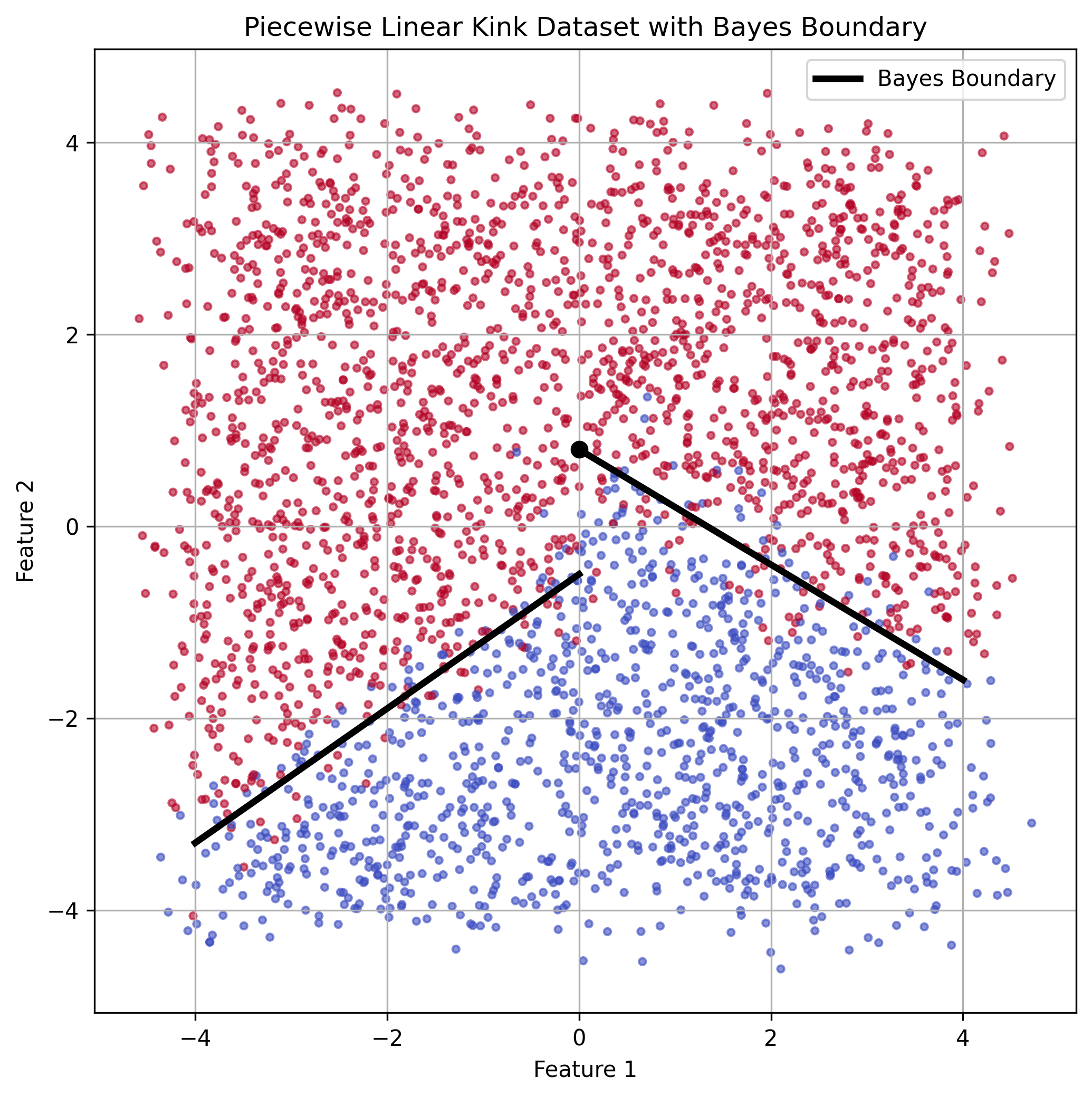}
\caption{Ground-truth decision boundary.}
\end{subfigure}
\hfill
\begin{subfigure}[!t]{0.43\textwidth}
\centering
\includegraphics[width=\textwidth]{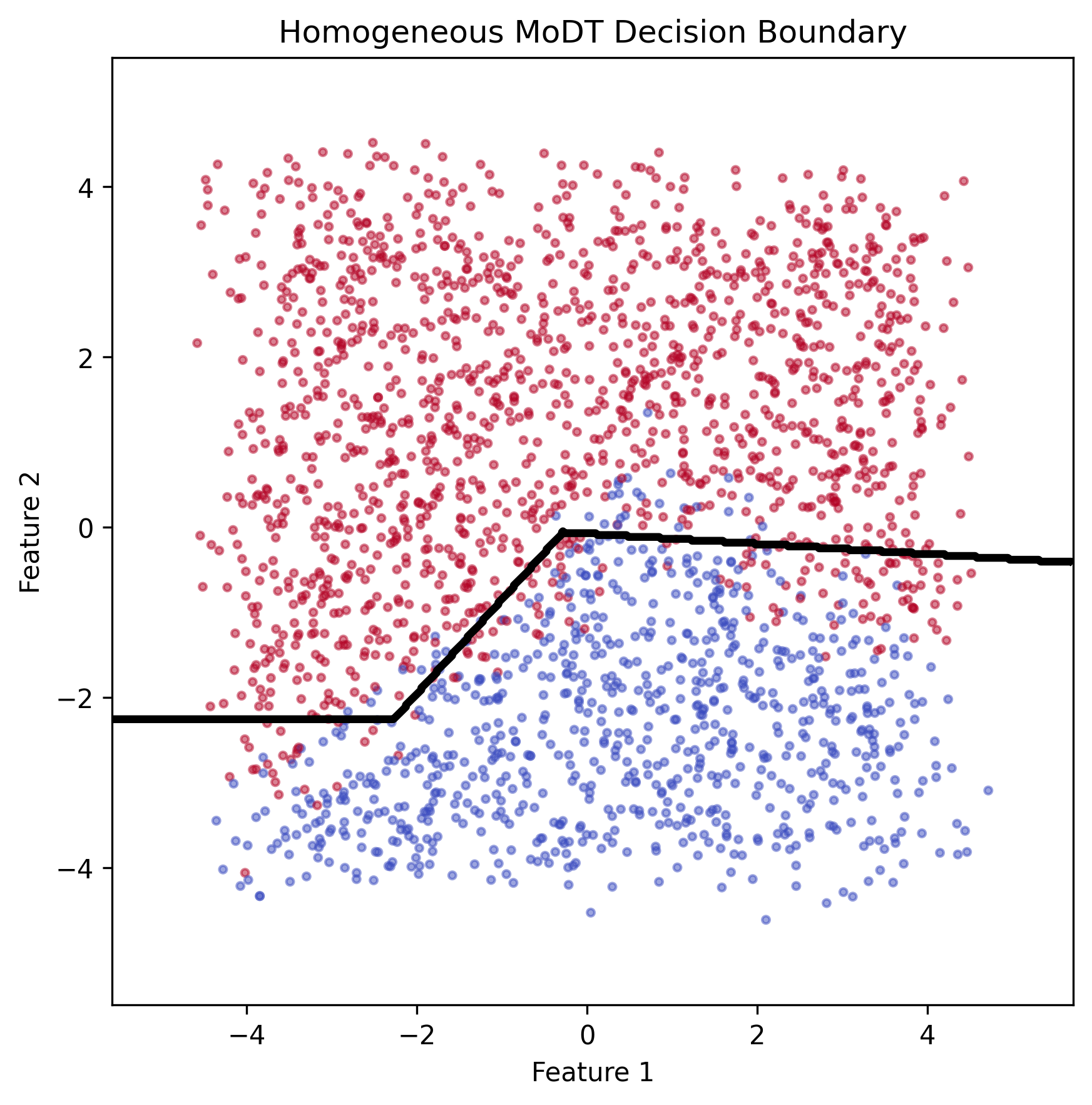}
\caption{Homogeneous MoDT decision boundary.}
\end{subfigure}

\vspace{0.5em}


\begin{subfigure}[!t]{0.5\textwidth}
\centering
\includegraphics[width=\textwidth]{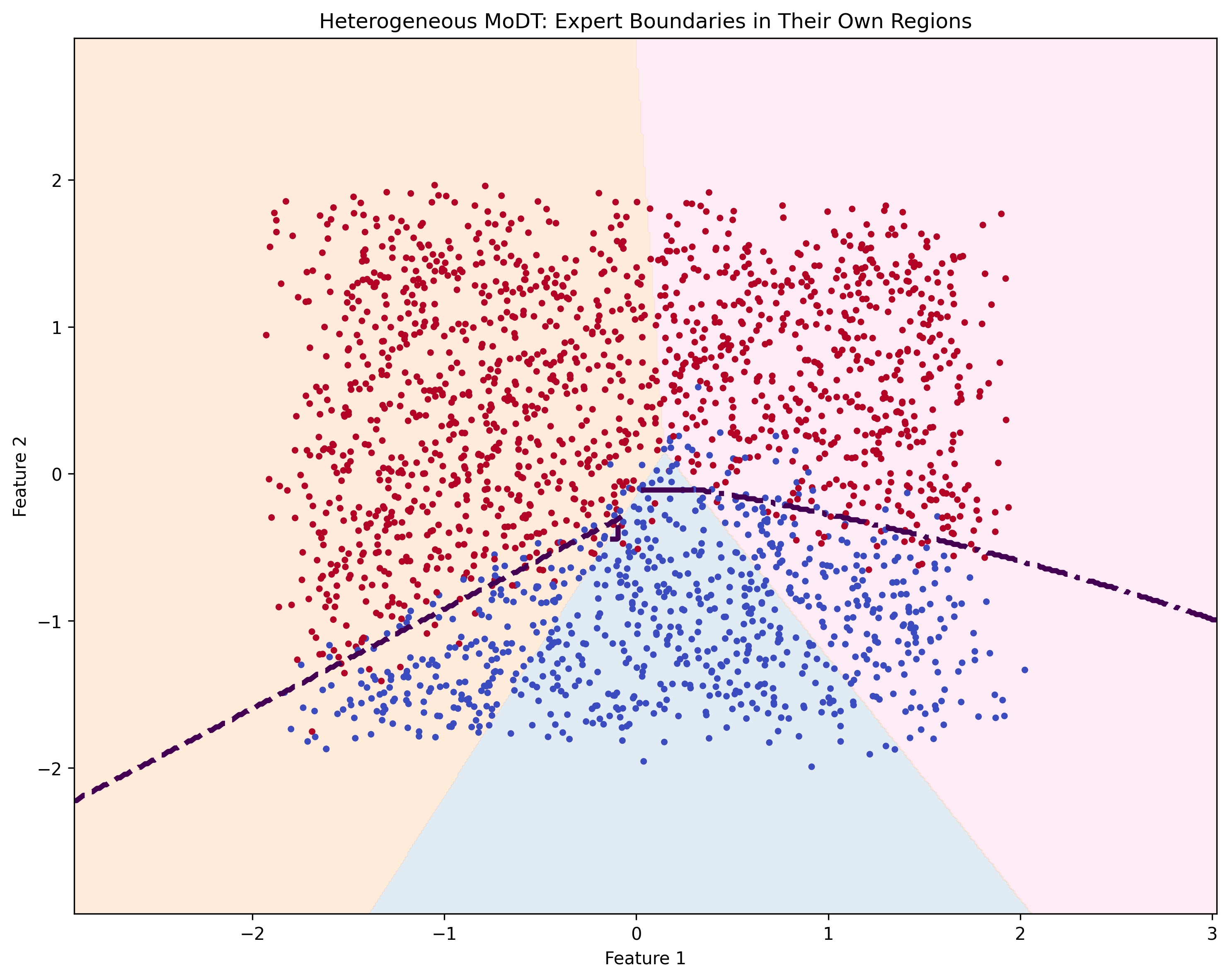}
\caption{Expert specialization learned by the heterogeneous Mixture Model.}
\end{subfigure}

\caption{
Comparison of decision boundaries and expert specialization on the \textbf{Piecewise Linear Kink} dataset. Subfigure (a) shows the ground-truth data-generating decision boundary consisting of two linear segments with opposite slopes that meet at a kink. Subfigure (b) presents the decision boundary learned by the  MoDT composed exclusively of decision-tree experts. Subfigure (c) visualizes the specialization behavior learned by the proposed heterogeneous mixture framework. The shaded background denotes the dominant expert selected by the gating mechanism, where beige corresponds to the linear SVM expert, pink corresponds to the QDA expert, and blue corresponds to the decision-tree expert. Red points denote observations from class~1, while blue points denote observations from class~0. Expert decision boundaries are displayed only within the regions where the corresponding expert receives dominant gating responsibility.
}
\label{fig:kink_specialization}
\end{figure*}

\begin{figure*}
\centering
\begin{subfigure}[t]{0.43\textwidth}
\centering
\includegraphics[width=\textwidth]{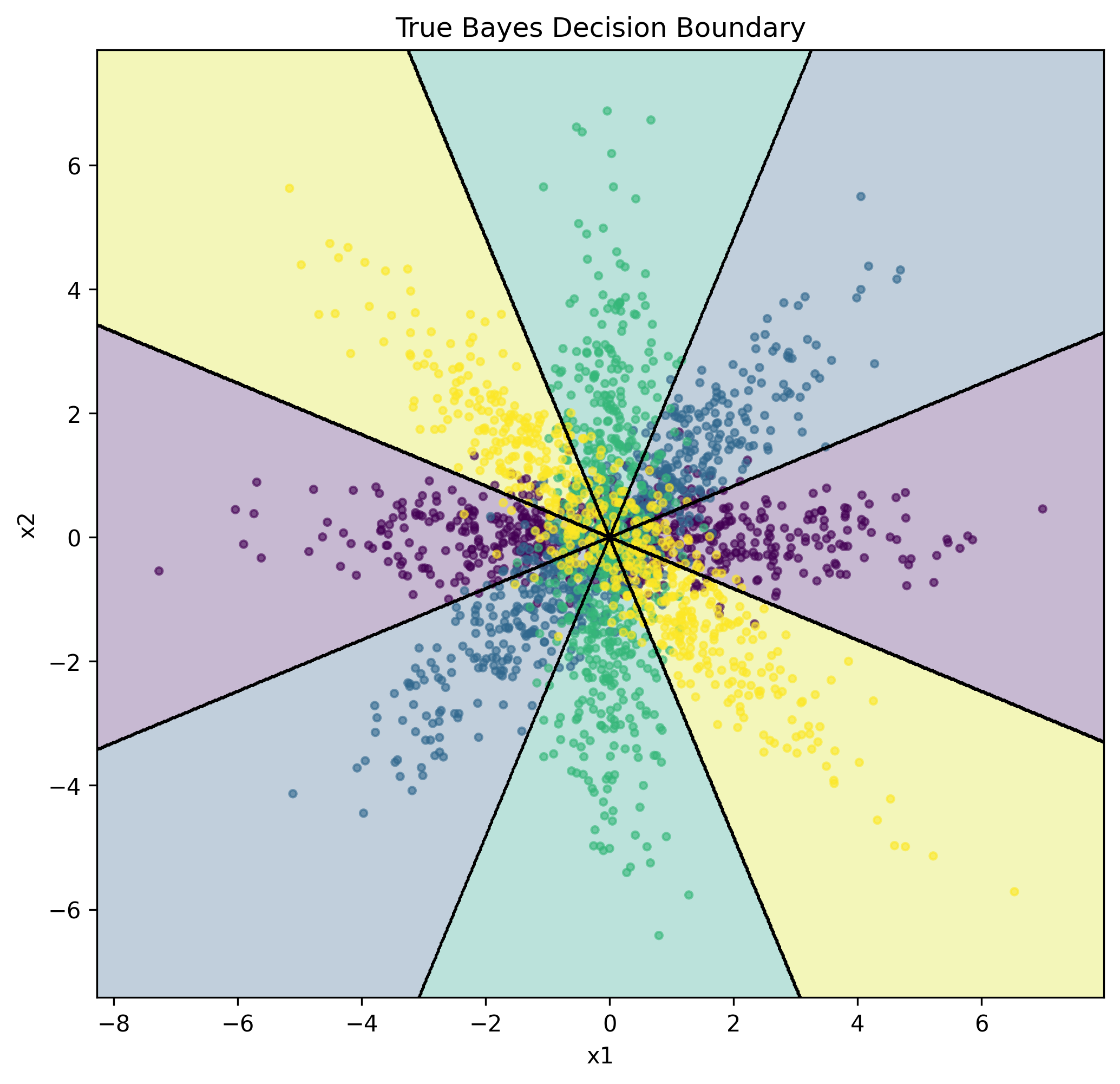}
\caption{Estimated Bayes-optimal boundary.}
\end{subfigure}
\hfill
\begin{subfigure}[t]{0.43\textwidth}
\centering
\includegraphics[width=\textwidth]{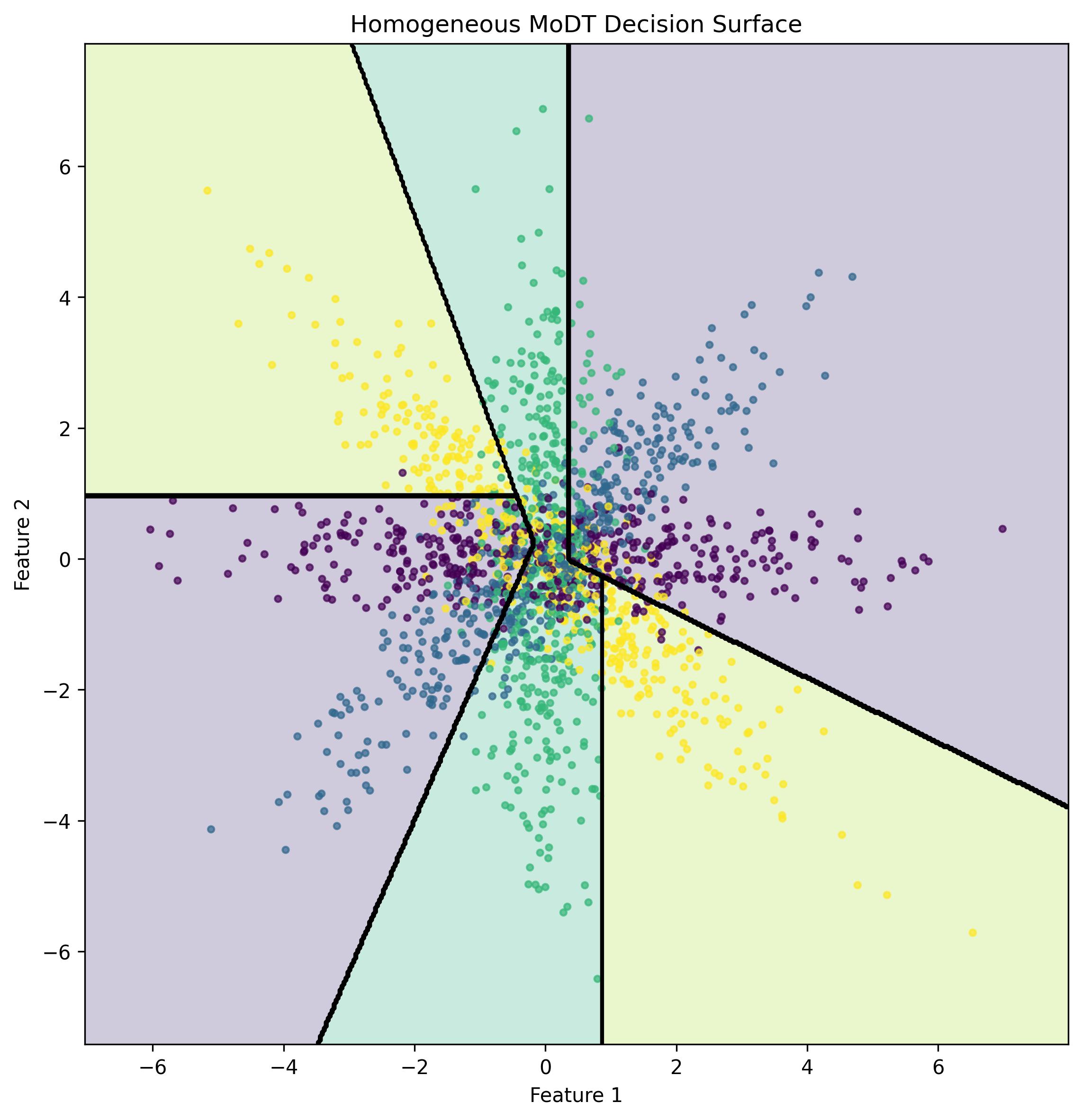}
\caption{Homogeneous MoDT decision boundary.}
\end{subfigure}

\vspace{0.5em}

\begin{subfigure}[t]{0.52\textwidth}
\centering
\includegraphics[width=\textwidth]{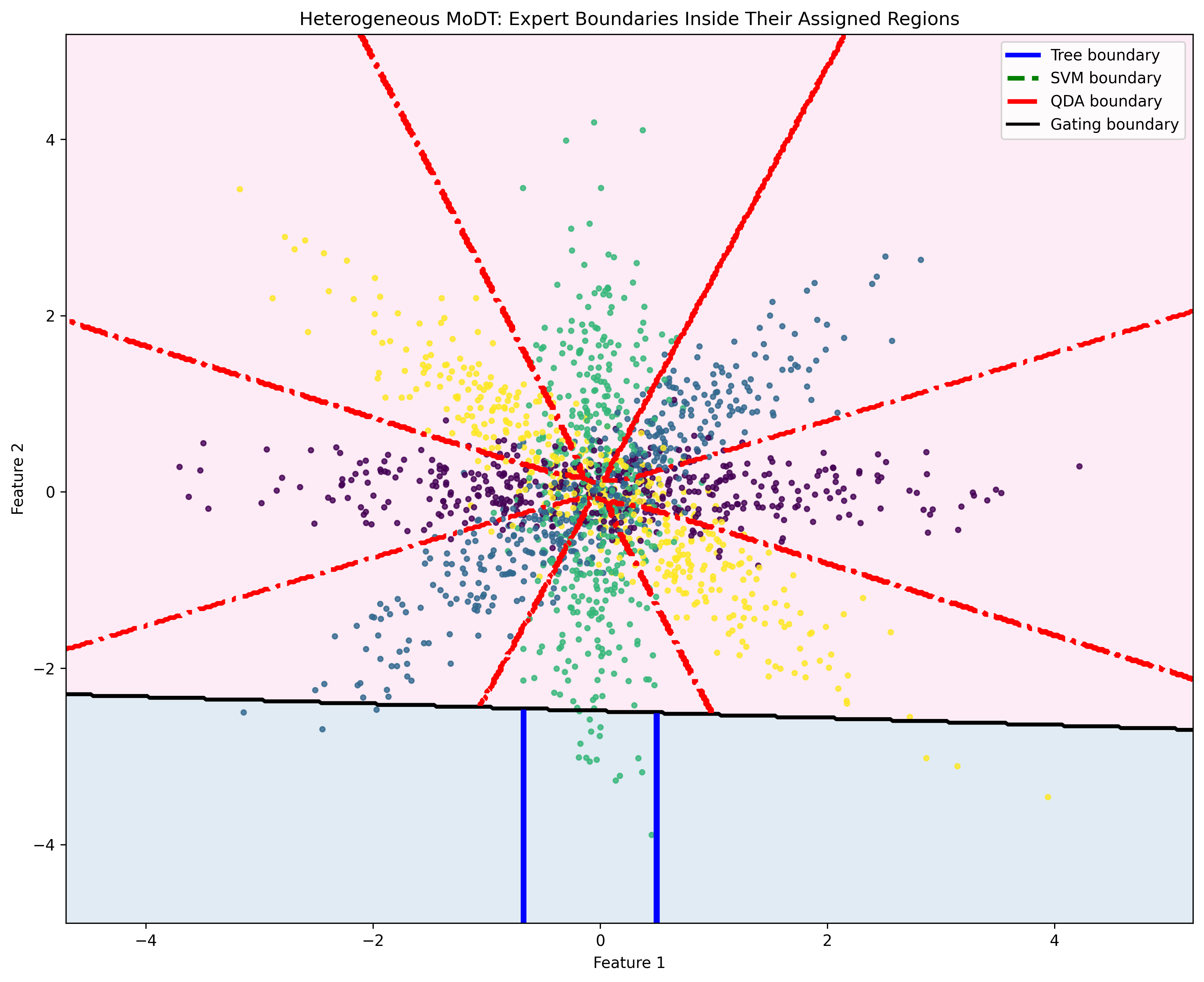}
\caption{Expert specialization learned by the heterogeneous Mixture Model.}
\end{subfigure}

\caption{
Comparison of decision boundaries and expert specialization on the \textbf{Rotated Gaussian Slabs} dataset. Subfigure (a) shows the estimated Bayes-optimal decision boundary obtained using quadratic discriminant analysis. Subfigure (b) presents the decision boundary learned by the MoDT composed exclusively of decision-tree experts. Subfigure (c) visualizes the specialization behavior learned by the proposed heterogeneous mixture framework. The shaded background denotes the dominant expert selected by the gating mechanism, where pink corresponds to the QDA expert, and blue corresponds to the decision-tree expert. Different colour of points denotes observations from different classes. Expert decision boundaries are displayed only within the regions where the corresponding expert receives dominant gating responsibility.
}
\label{fig:AG_specialization}
\end{figure*}

\paragraph{Qualitative Analysis of Decision Boundaries and Expert Specialization.}

Figures~\ref{fig:kink_specialization} and~\ref{fig:AG_specialization} compare the ground-truth (or Bayes-optimal) decision boundaries, the decision boundaries learned by the MoDT, and the specialization patterns learned by the proposed heterogeneous Mixture of Experts framework. For the Piecewise Linear Kink dataset, the ground-truth boundary consists of two linear segments with opposite slopes that meet at a kink, creating a decision structure that cannot be represented by a single global linear classifier. The MoDT successfully captures the presence of the kink but approximates the underlying geometry through a collection of piecewise axis-aligned partitions induced by tree experts. In contrast, the heterogeneous Mixture produces a more interpretable decomposition of the feature space by assigning the left branch primarily to the SVM expert, the right branch to the QDA expert, and a lower-central region to the decision tree expert. Consequently, each expert contributes a meaningful local decision boundary that reflects its underlying inductive bias, yielding a decomposition that more closely aligns with the geometry of the data-generating process.

A different specialization pattern emerges on the Multiclass Anisotropic
Gaussian dataset, whose Bayes-optimal classifier is effectively linear.
Although a linear support vector machine is also capable of representing such
a decision boundary, the heterogeneous mixture assigns the overwhelming
majority of the posterior responsibility to the QDA expert. This behaviour is
consistent with the underlying data-generating process. The class-conditional
distributions possess approximately equal covariance structures, under which
the quadratic discriminant function reduces to linear discriminant
analysis (LDA), yielding an essentially linear decision boundary. Consequently,
the QDA expert naturally recovers the optimal decision rule while retaining the
ability to model more general quadratic boundaries if required. The gating
mechanism therefore has little incentive to allocate substantial responsibility
to the linear SVM expert.

In contrast, the original MoDT is restricted to decision tree experts and must
approximate the underlying linear boundary through a collection of axis-aligned
partitions. Although the gating mechanism combines multiple tree experts, the
resulting piecewise approximation cannot recover the globally linear Bayes
decision boundary as faithfully, leading to comparatively lower predictive
performance.

Since the average expert usage assigned to the QDA expert exceeds
90\%, we additionally evaluate the corresponding standalone QDA classifier.
Figure~\ref{fig:AG_specialization} presents both the Bayes-optimal (QDA)
decision boundary and the boundary learned by the proposed heterogeneous
mixture. Although the two boundaries are nearly identical, the comparison
serves a different purpose. It demonstrates that the gating mechanism
successfully identifies the statistically appropriate expert family and
effectively reduces the heterogeneous mixture to the correct local model.
Consequently, the figure illustrates not only accurate prediction but also the
ability of the proposed framework to perform meaningful expert specialization
through adaptive model selection.

These observations highlight a principal advantage of the proposed framework.
Rather than enforcing a homogeneous hypothesis class across all experts, the
heterogeneous mixture is capable of selecting the inductive bias most
appropriate for different regions of the feature space. In this example, the
gating mechanism correctly recognises that a covariance-based discriminative
model provides the most faithful representation of the underlying
data-generating process, thereby recovering the Bayes-optimal decision boundary
while preserving the interpretability of the overall mixture architecture.
Taken together, these examples demonstrate that the proposed heterogeneous mixture framework does not enforce uniform participation among experts; rather, it adaptively selects and combines heterogeneous inductive biases according to the geometric and statistical characteristics of different regions of the feature space. The resulting decomposition improves interpretability while providing insight into how different expert classes collaborate to solve classification tasks of varying complexity.
\subsection{Application on Real-World Datasets}

\begin{table*}[!t]
  \centering
  \caption{Summary of real-world datasets.}
  \label{tab:dataset_summary1}
  \resizebox{\textwidth}{!}{%
    \begin{tabular}{lcccp{6cm}}
      \toprule
      \textbf{Dataset Name} & \textbf{\# Data Points} & \textbf{\# Classes}
        & \textbf{\# Features} & \textbf{Description} \\
      \midrule
      Adult         & 45{,}222 & 2 & 103 & Census-based dataset used for income prediction. \\
      GYM           & 2{,}000  & 2 & 4   & Reinforcement learning state-action classification dataset. \\
      Breast Cancer & 569      & 2 & 10  & Predict breast cancer based on properties extracted from images. \\
      Iris          & 150      & 3 & 4   & Classify iris flower species based on sepal and petal measurements. \\
      Wine (Chemical) & 178 & 3 & 13 & Classify wine cultivars using chemical composition measurements. \\
      Congressional Voting Records & 435 & 2 & 16 & Classify political party from congressional roll-call voting records. \\
      Connectionist Bench & 208 & 2 & 60 & Classify sonar signals as reflections from mines or rocks using acoustic measurements. \\
      Rice (Osmancik vs Cammeo) & 3810 & 2 & 7 & Distinguish Osmancik and Cammeo rice varieties using morphological characteristics extracted from grain images. \\
      \bottomrule
    \end{tabular}%
  }
\end{table*}

\begin{table*}[!t]
  \centering
  \caption{Comparison of models across real-world datasets (mean accuracy $\pm$ std).}
  \label{tab:model_comparison1}
  \resizebox{\textwidth}{!}{%
    \begin{tabular}{lcccccc}
      \toprule
      \multirow{2}{*}{Dataset}
        & \multirow{2}{*}{Random Forest}
        & \multicolumn{2}{c}{MoDT}
        & \multicolumn{2}{c}{Hetero-Mix Expert}
        & \multirow{2}{*}{Avg.\ Expert Usage (\%)} \\
      \cmidrule(lr){3-4}\cmidrule(lr){5-6}
        & & Depth 2 & Depth 3 & Depth 2 & Depth 3 & \\
      \midrule
      Iris          & $0.94\pm0.02$ & $0.93\pm0.03$ & $0.94\pm0.03$ & $\mathbf{0.95\pm0.03}$ & $0.93\pm0.02$ & DT 62.1, QDA 37.9 \\
      Breast Cancer & $\mathbf{0.94\pm0.02}$ & $0.92\pm0.02$ & $0.93\pm0.02$ & $0.91\pm0.02$ & $0.92\pm0.02$ & DT 52.5, SVM 0.1, QDA 47.4 \\
      Gym           & $\mathbf{0.94\pm0.01}$ & $\mathbf{0.94\pm0.01}$ & $0.93\pm0.02$ & $0.93\pm0.01$ & $\mathbf{0.94\pm0.01}$ & DT 100 \\
      Adult         & $0.789$       & $0.783$       & $0.849$       & $\mathbf{0.850}$       & $0.838$       & DT 7.9, SVM 90.2, QDA 1.9 \\
      Wine (Chemical) & $\mathbf{0.99\pm0.01}$ & $0.93\pm0.03$ & $0.92\pm0.03$ & $0.94\pm0.03$ & $0.94\pm0.02$ & DT 49.6, QDA 50.4 \\
      Congressional Voting Records & $\mathbf{0.96\pm0.01}$ & $0.94\pm0.01$ & $0.94\pm0.01$ & $0.95\pm0.01$ & $0.95\pm0.01$ & DT 47.5, SVM 2.1, QDA 50.3\\
      Connectionist  & $\mathbf{0.80\pm0.01}$ & $0.69\pm0.06$ & $0.72\pm0.06$  & $0.71\pm0.05$ & $0.73\pm0.09$ & DT 32.6, SVM 0.3, QDA 67.2 \\
      Rice (Osmancik vs Cammeo) & $0.92\pm0.01$ & $0.92\pm0.01$ & $0.92\pm0.01$ & $\mathbf{0.93\pm0.01}$ & $0.92\pm0.01$ & DT 42.2, SVM 23.6, QDA 34.2 \\
      
      \bottomrule
    \end{tabular}%
  }
\end{table*}

\begin{figure*}
  \centering
  \includegraphics[width=\textwidth]{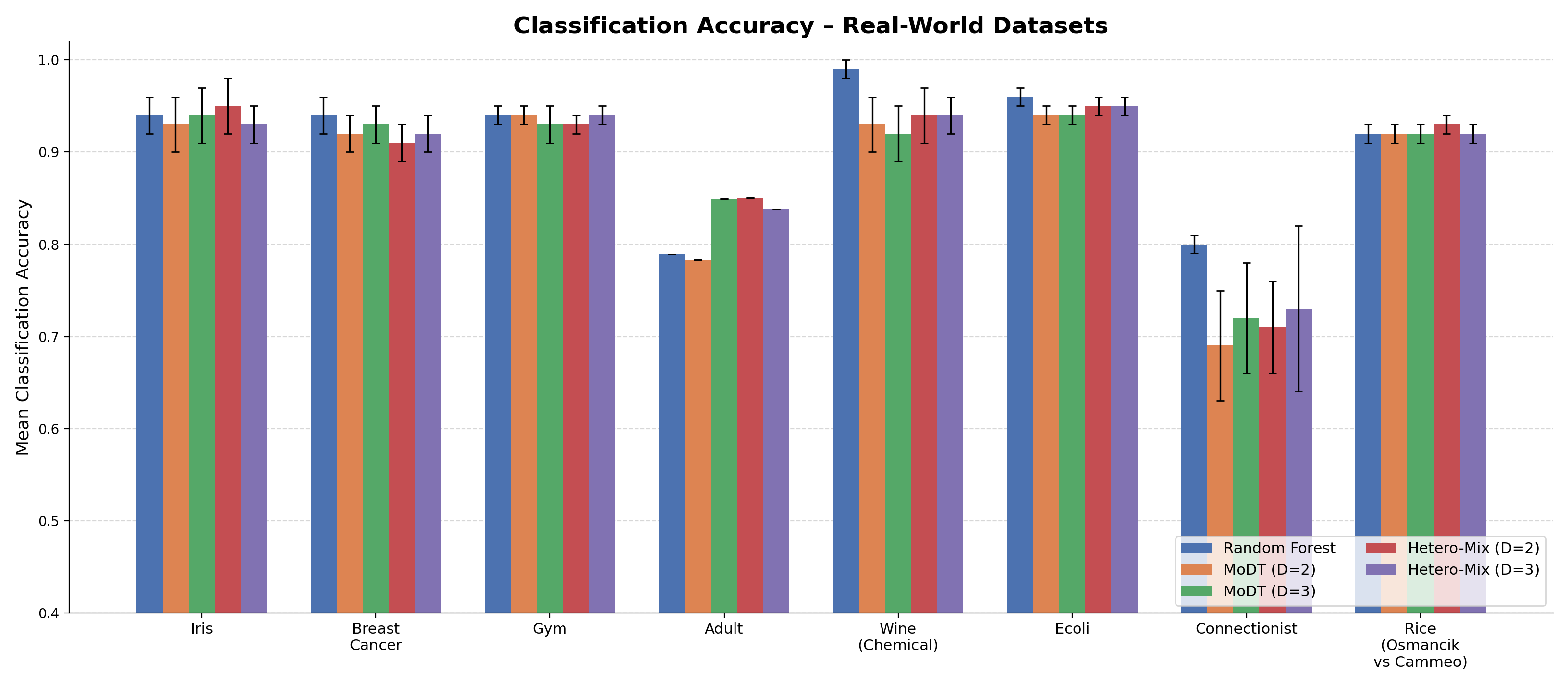}
  \caption{Mean classification accuracy on real-world datasets. Error bars indicate one
           standard deviation over ten random train-test splits. Results are shown for
           Random Forest, MoDT (depth~2 and~3), and the proposed Hetero-Mix (depth~2
           and~3).}
  \label{fig:realworld_accuracy}
\end{figure*}

\begin{figure*}
  \centering
  \includegraphics[width=\textwidth]{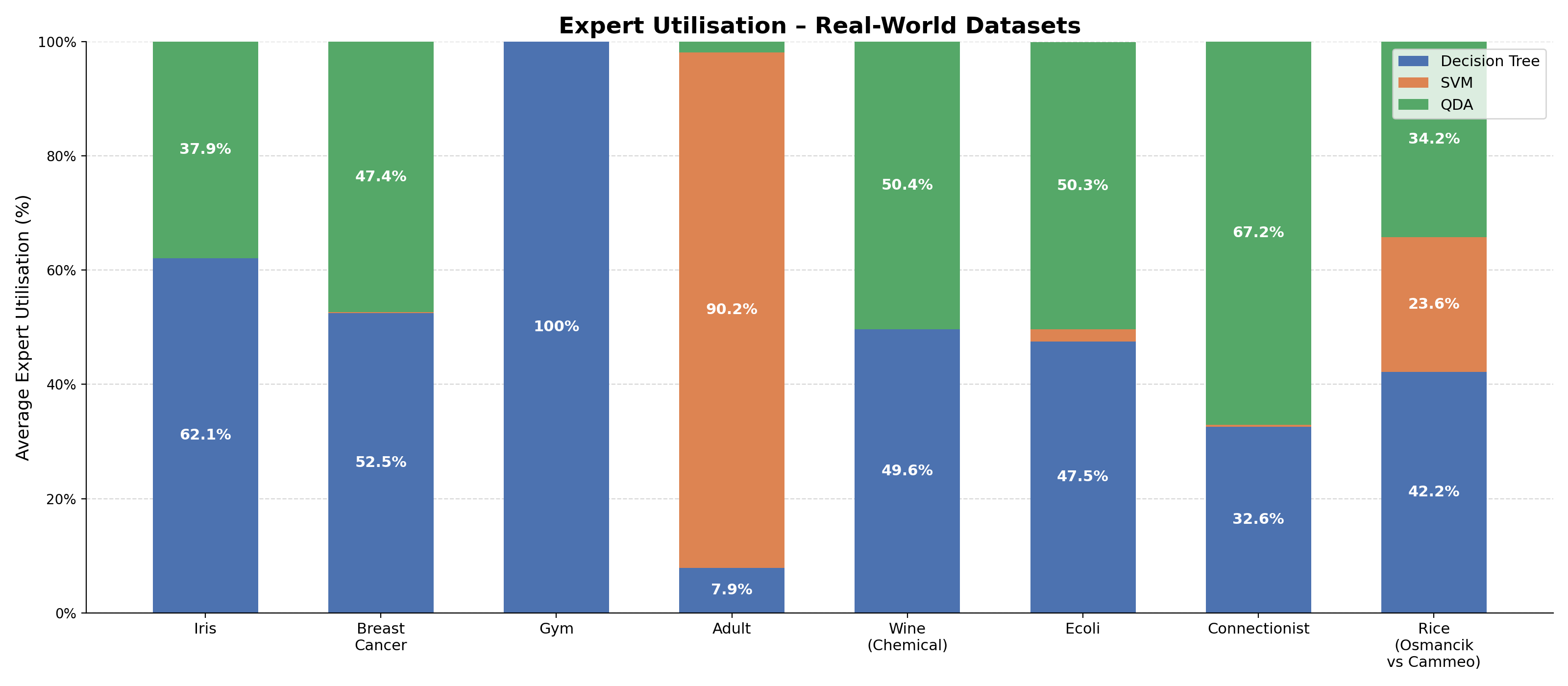}
  \caption{Average expert utilisation (\%) on real-world datasets for the proposed
           heterogeneous mixture model. Each stacked bar shows the proportion of samples
           assigned to the Decision Tree, SVM, and QDA experts by the gating mechanism,
           averaged over ten train-test splits.}
  \label{fig:realworld_experts}
\end{figure*}

To evaluate the proposed framework on practical classification problems, we
consider eight widely used benchmark datasets drawn from diverse application
domains. The Adult~\cite{adult_2}, Breast Cancer, Gym, and Iris~\cite{iris_53} datasets were obtained from
the publicly available implementation accompanying the original Mixture of
Decision Trees (MoDT) framework~\cite{bru}, while the Connectionist (Sonar)~\cite{connectionist_bench_sonar_mines_vs_rocks_151},
Wine~\cite{wine_109}, Congressional Voting Records, and Rice~\cite{rice_cammeo_and_osmancik_545} datasets were obtained from the UCI Machine Learning
Repository~\cite{uci,Dua2017UCI}. Table~\ref{tab:dataset_summary1} gives the descriptions of the datasets analyzed while
Table~\ref{tab:model_comparison1} and Figure~\ref{fig:realworld_accuracy} summarize the
classification performance of the proposed heterogeneous mixture framework on several
real-world benchmark datasets. Figure~\ref{fig:realworld_experts} illustrates the
corresponding expert utilisation patterns. In general, the
heterogeneous mixture model provides competitive performance compared to Random Forests and
the MoDT, while also providing interpretable expert specialization
behavior. On several datasets (e.g., Adult, Iris, Rice), the heterogeneous model performs as well
or slightly better than the MoDT architecture, which suggests that a mixture of
heterogeneous experts can lead to better generalization across datasets with different
underlying structures. Moreover, the expert utilization figures show interesting
specialization trends. For example, the Adult dataset is mostly dominated by the SVM
expert, indicating that there are roughly linearly separable regions in the feature space,
whereas datasets such as Gym show almost total dominance of a single expert,
indicating that some datasets are naturally adapted to particular expert structures.
Conversely, the more balanced expert utilization in datasets such as Rice
and Breast Cancer suggests that collaborative interaction among heterogeneous experts may
play a role in predictive performance in more complex classification scenarios.

\section{Conclusion and Future Research}
\label{sec:conc_fut_res}

This work develops Heterogeneous Mixture of Experts (Hetero-MoE) as a framework for pattern classification in which the predictive mechanism can adapt not only through its parameters and spatial allocation of observations, but also through the choice of the underlying hypothesis class. The distinction is fundamental. In a conventional homogeneous mixture, the gating mechanism determines which expert is emphasized, but the experts themselves remain members of the same model family. In Hetero-MoE, the local representation is allowed to change with the structure of the problem: a decision tree can describe an axis-aligned rule, a linear support vector machine can represent a locally hyperplanar separation, and quadratic discriminant analysis can capture covariance-dependent class structure. The resulting model therefore treats the selection of an inductive bias as part of the pattern-recognition problem itself. Rather than asking only how accurately a decision boundary can be approximated, the framework also asks which interpretable form of decision boundary is appropriate in a given region of the feature space.

The principal strength of the proposed approach is consequently its ability to expose a level of model adaptation that is usually hidden inside a single predictive architecture. A flexible homogeneous learner may approximate several geometries simultaneously, but the resulting representation does not necessarily reveal which modeling assumption is responsible for the prediction in a particular region. Hetero-MoE makes this distinction explicit. The gating mechanism determines the relative contribution of expert families, while the experts retain their own recognizable statistical structure. The resulting decomposition provides a natural bridge between mixture modeling, local model selection, and interpretable pattern recognition. In this sense, the proposed framework is not simply a larger ensemble of classifiers; it makes the choice of predictive representation itself an object of estimation.

The empirical results demonstrate both the usefulness and the limits of this additional flexibility. Hetero-MoE improves upon or matches homogeneous MoDT on several problems whose local geometry is heterogeneous or covariance-sensitive, while its expert-utilization patterns often correspond to recognizable properties of the underlying recognition problem. In the Piecewise Linear Kink experiment, for example, the simultaneous participation of decision-tree, SVM, and QDA experts provides a local decomposition of a globally heterogeneous boundary rather than an approximation based exclusively on axis-aligned partitions. Similar specialization patterns are observed in the real-world experiments: the Adult data are predominantly associated with the SVM expert, Gym is effectively represented by the decision-tree expert, and datasets such as Rice and Breast Cancer exhibit more balanced utilization across expert families. These results demonstrate that expert utilization can serve not only as an optimization outcome but also as a descriptive representation of the statistical structure identified by the fitted classifier.

At the same time, the experiments show that heterogeneity is not universally beneficial. There are datasets on which Random Forest or homogeneous MoDT remains competitive or superior to Hetero-MoE. This observation is an important part of the contribution rather than a weakness to be concealed. It indicates that introducing additional hypothesis classes does not automatically improve generalization; the value of heterogeneity depends on the relationship between the data geometry, the available expert families, the routing mechanism, and the amount of information available to estimate the resulting local models. Hetero-MoE should therefore be viewed as an additional modeling degree of freedom rather than as a universal replacement for homogeneous classifiers. The more informative question for practitioners becomes not simply `Which classifier should be used?'' but, when the data justify such a decomposition, `Which classifier should be used where?''

This positioning is complementary to recent developments in interpretable Pattern Recognition and mixture-of-experts modeling. Recent Pattern Recognition studies have considered intrinsically interpretable approaches based on prototypes, neural decision trees, and constrained or monotonic architectures~\cite{wen2025cdpnet,luo2024ncart,wargnier2025}, while recent MoE research has investigated intrinsic interpretability through structured expert architectures and routing~\cite{yang2025}. More recent work in \emph{Pattern Recognition} has also demonstrated the predictive value of heterogeneous neural experts combined through adaptive routing in fine-grained visual classification~\cite{yang2026fgmoe}. These approaches establish that specialization and routing remain active themes in contemporary pattern recognition, but they address settings that differ materially from the present one. In particular, Hetero-MoE combines heterogeneous, intrinsically interpretable statistical learners with distinct inductive biases inside a common probabilistic conditional model. Its principal distinction is therefore not that heterogeneous routing itself is new, but that \emph{heterogeneous classical hypothesis classes can be treated as explicitly selectable local recognition mechanisms while preserving probabilistic interpretability}.

Accordingly, the results should not be interpreted as establishing universal state-of-the-art predictive accuracy across contemporary classification methods. Such a claim is neither required by the proposed framework nor supported uniformly by the experiments. Rather, the empirical evidence shows that Hetero-MoE can remain competitive with strong predictive baselines while providing information that is structurally different from that supplied by black-box ensembles or homogeneous interpretable mixtures. A Random Forest can be highly effective without exposing a region-wise decomposition into distinct statistical mechanisms, while homogeneous MoDT retains transparency but restricts every expert to the same hypothesis class. Hetero-MoE occupies the intermediate space between these two objectives by allowing predictive competition across fundamentally different interpretable mechanisms and by exposing the resulting allocation through the gating probabilities and expert responsibilities.

For practitioners in pattern recognition, this additional information can be useful at several stages of model development. Expert utilization can provide a form of structural diagnosis: near-complete dominance by one expert may indicate that the data are adequately represented by a simpler homogeneous model, whereas persistent use of several experts may indicate heterogeneous local structure. The learned specialization can also guide subsequent model development by revealing which kinds of decision mechanisms appear to be relevant before a practitioner adopts a substantially more complex black-box architecture. Moreover, the use of recognizable expert families provides a natural language for communicating model behavior to domain experts. A prediction can be associated with a tree-based rule, a linear separation, or a covariance-driven discriminant mechanism rather than only with an opaque latent representation. Because the individual experts remain independently interpretable, local portions of the fitted decision function can also be inspected, audited, or replaced without requiring the entire predictive architecture to be discarded.

A related strength is that the proposed framework separates three forms of adaptation that are often conflated in predictive modeling: adaptation of model parameters, adaptation of the spatial allocation of observations, and adaptation of the hypothesis class itself. Classical homogeneous mixtures provide the first two forms, whereas Hetero-MoE introduces the third. This creates a conceptual link between model selection and mixture modeling: rather than choosing a single global hypothesis class before estimation, the model allows the data to determine which interpretable predictive paradigm should govern different parts of the feature space. The resulting framework therefore provides a mechanism for representing heterogeneity in the \emph{form} of a classifier rather than only in its fitted parameters.

The probabilistic formulation provides an additional strength. Expert responsibilities are obtained from the posterior distribution of the latent allocation variable rather than from an externally imposed similarity score or an independently defined heuristic confidence measure. Consequently, expert utilization has a direct probabilistic interpretation within the fitted conditional mixture. The calibration of the SVM outputs is important in this respect because it allows a margin-based classifier to participate in the same probabilistic mixture as the other experts. The generalized EM formulation consequently establishes a coherent connection between local prediction, posterior responsibility, and expert allocation. The accompanying monotone-ascent result for the gating auxiliary objective further shows that the regression-based gating correction is not merely a computational heuristic, but admits a formal generalized-EM interpretation under the stated learning-rate conditions.

The interpretability offered by Hetero-MoE should nevertheless be understood in a precise sense. The present study evaluates interpretability primarily through the intrinsic transparency of the expert families, explicit gating probabilities, expert-utilization statistics, and visualization of local decision boundaries. These components provide structural information that is unavailable from a purely black-box predictor, but they do not constitute a complete human-centred assessment of whether the resulting explanations are useful, faithful, stable, or cognitively efficient. Recent work has emphasized the need to evaluate explanation quality explicitly rather than infer it solely from model architecture~\cite{explanationstrength2025}. A useful continuation of the present work is therefore not to establish interpretability as a new problem, but to determine how well the specific explanation furnished by \emph{expert-family selection} communicates the behavior of a heterogeneous classifier to practitioners and domain experts.

The present work nevertheless has several limitations. First, the performance and specialization behavior of the framework depend on the collection of expert families made available to the gating mechanism. If an appropriate expert class is absent, the model may still experience local model mismatch despite its heterogeneous structure. Second, the framework does not guarantee balanced utilization of all experts. Indeed, for some datasets the gating mechanism naturally collapses toward a single dominant expert, effectively recovering a homogeneous model. While such behavior is often informative and may itself reveal the most suitable inductive bias for the dataset, it reduces the practical benefits of maintaining a larger heterogeneous architecture.

A second limitation concerns the expert routing mechanism. The proposed framework inherits the probabilistic softmax gating formulation from the classical mixture-of-experts architecture~\cite{jacobs1991,jordan1994} while preserving the interpretable routing mechanism employed in MoDT~\cite{bru}. Although this provides computational efficiency and a coherent probabilistic interpretation of expert responsibilities, a linear gating function may become restrictive when the true allocation boundaries are highly nonlinear, and likelihood-based responsibilities may not always yield the most informative or stable expert specialization. More flexible routing mechanisms are already well established in contemporary MoE systems; consequently, the unresolved question in the present setting is narrower: how can routing be made nonlinear or adaptive \emph{without losing the intrinsic interpretability of heterogeneous statistical experts}? In particular, tree-based, additive, or other structured gates could be investigated so that the routing boundaries themselves remain interpretable while accommodating substantially more complex allocation geometries.

The dependence on a prespecified expert dictionary is particularly consequential for the broader vision of adaptive pattern recognition. At present, the framework can choose among the models supplied by the practitioner, but it cannot discover an entirely new hypothesis class when none of the available experts is appropriate. While model-selection procedures for the number of experts and other aspects of MoE architecture have been studied in the broader literature, the unresolved question here is more specific: how should a heterogeneous conditional mixture choose \emph{which statistical learning paradigms} should be present in its expert dictionary when those paradigms have different parameterizations, dimensions, optimization procedures, and interpretability properties? This problem is distinct from simply selecting the number of experts. A future model could, for example, jointly determine whether a decision-tree, margin-based, discriminant, additive, or prototype-based expert is warranted and whether its inclusion provides sufficient predictive benefit to justify the additional complexity. Developing such a parsimonious expert-selection mechanism would turn the current practitioner-specified dictionary into a more fully data-adaptive component of the model.

The present work also leaves several important theoretical questions unanswered. Existing work has established substantial asymptotic theory for conventional mixtures-of-experts, including nonstandard behavior associated with latent expert allocation and the estimation of the number of experts~\cite{jiang2011asymptotic}. Recent Bayesian developments have further studied posterior contraction, parameter estimation, identifiability, and model selection for softmax-gated mixtures-of-experts~\cite{bariletto2026}. These results substantially reduce the scope of what can reasonably be described as theoretically unexplored in MoE methodology. The unresolved issue for Hetero-MoE is instead the extension of such theory to a conditional mixture whose experts belong to fundamentally different statistical families. In the present model, the components may have different parameter dimensions, different objective functions, different regularity properties, and different notions of local complexity. It therefore remains to be determined which forms of identifiability, consistency, asymptotic normality, convergence rate, and model-selection guarantees continue to hold when heterogeneous interpretable experts coexist within the same conditional mixture. Establishing such results would provide a theoretical foundation specifically for heterogeneous statistical mixtures rather than for mixtures whose components are instances of a common parametric family.

The empirical results also reveal an interesting phenomenon that deserves further investigation. Although the proposed heterogeneous framework frequently outperforms or matches the homogeneous MoDT architecture~\cite{bru}, there also exist datasets for which the original homogeneous model achieves superior predictive performance. This observation suggests that heterogeneity is not universally advantageous and raises a more specific unresolved question concerning the relationship between data geometry, expert diversity, and model complexity. General MoE research already studies expert diversity, routing balance, and specialization; the question that remains open in the present setting is whether one can characterize, in terms of properties of the underlying classification problem, when the additional approximation flexibility obtained by heterogeneous hypothesis classes outweighs the additional estimation and routing complexity. Such a characterization would provide a principled basis for deciding when Hetero-MoE should be preferred to a homogeneous MoDT rather than relying on empirical trial alone.

A particularly promising empirical direction is therefore to move from asking whether Hetero-MoE performs well on a collection of datasets to asking \emph{which measurable properties of a dataset predict the value of heterogeneous modeling}. The current experiments already show substantial variation in expert utilization. A systematic follow-up could relate the gain of Hetero-MoE over homogeneous MoDT to quantities such as utilization entropy, the spatial separation of expert responsibilities, the complexity of local decision boundaries, and measures of local class-conditional geometry. The purpose would not be simply to produce another descriptive visualization, but to determine whether these quantities can serve as empirical indicators of the conditions under which heterogeneous modeling is warranted. Such an analysis would directly operationalize the question of \emph{when heterogeneity is beneficial}, which remains more specific to the present framework than the broader questions of expert diversity or routing optimization studied in contemporary MoE research.

The possibility of expert collapse also merits a more nuanced interpretation. In many modern MoE systems, severe imbalance in expert utilization is treated as a practical routing problem because it can produce inefficient computation or poorly trained experts. In Hetero-MoE, however, concentration on a single expert can have a second interpretation: it may indicate that the data provide little evidence for multiple inductive biases. The unresolved problem is therefore not simply how to force more balanced utilization, but how to distinguish undesirable optimization-induced collapse from statistically meaningful evidence that a homogeneous model is adequate. A principled criterion for making this distinction could allow expert collapse to become a model-simplification signal rather than merely a failure mode. This would also provide a natural connection between routing, model selection, and interpretability.

The framework can be expanded beyond the specific collection of experts considered in this study. The same architecture can naturally incorporate generalized additive models, sparse linear models, probabilistic graphical models, Gaussian process experts, robust classifiers, prototype-based learners, or other interpretable statistical procedures. The important unresolved problem is not whether such models can be inserted mechanically into the mixture, but how their differing notions of complexity, regularization, uncertainty, and interpretability should be compared on a common basis during joint estimation. In particular, future work could investigate expert libraries tailored to particular Pattern Recognition domains and develop criteria by which additional expert families are admitted only when they offer a demonstrable and interpretable improvement over the existing dictionary.

A second major direction is therefore the development of more flexible but still interpretable gating mechanisms. Although the current linear softmax gate is attractive for its simplicity, computational efficiency, and transparent parameterization, it may be insufficient when the regions in which different experts are appropriate are themselves highly nonlinear. Existing MoE literature already contains many powerful nonlinear and spatially adaptive routing mechanisms~\cite{yang2026fgmoe}; the more specific question for Hetero-MoE is how to obtain comparable routing flexibility while retaining an explicit statistical interpretation of both the gate and the expert assignment. Tree-based gates, sparse additive gates, piecewise-linear gates, or other structured routing functions are therefore natural candidates. In parallel, the routing objective could be extended to incorporate predictive uncertainty, calibration, expert complementarity, and stability. The unresolved issue is how these criteria should be combined without allowing routing to become so complex that the principal interpretability benefit of the heterogeneous architecture is lost.

The theoretical development of such extensions would also be valuable. The monotone-ascent result established here concerns the gating auxiliary objective and provides an optimization guarantee under specified learning-rate conditions; it does not establish global convergence or statistical optimality of the complete heterogeneous estimator. A more comprehensive theory would ideally connect the geometry of the expert classes to approximation error, estimation error, routing complexity, and identifiability, and would characterize the conditions under which heterogeneous mixtures offer a genuine statistical advantage over homogeneous mixtures. In particular, it remains to be determined whether one can formulate oracle inequalities, risk comparisons, or data-dependent selection criteria that quantify when the gain from allowing multiple inductive biases justifies the additional complexity of the heterogeneous architecture.

Finally, the broader significance of the proposed framework extends beyond the specific collection of experts considered in this work. The central idea underlying mixture-of-experts models~\cite{jacobs1991,jordan1994} is that different regions of the feature space may require different predictive models. The present work extends this philosophy by allowing those local models to originate from fundamentally different statistical learning paradigms rather than a common hypothesis class. Consequently, the proposed methodology is not restricted to decision trees, linear support vector machines, and quadratic discriminant analysis. The same framework can naturally incorporate generalized additive models, sparse linear models, probabilistic graphical models, Gaussian process experts, robust classifiers, or other interpretable statistical learning procedures. We therefore view the proposed heterogeneous mixture-of-experts framework not merely as an extension of MoDT, but as a general methodology for combining heterogeneous statistical models within a unified probabilistic architecture. The central research opportunity is now to understand the statistical principles governing such heterogeneous specialization, to determine when it is beneficial, and to develop routing and selection mechanisms that preserve interpretability while admitting greater structural flexibility.

Taken together, the contribution of this work is broader than the particular empirical gains obtained by combining three expert families. The deeper contribution is to make \emph{model heterogeneity itself an interpretable object of learning}. Instead of treating the choice between a tree, a linear classifier, or a covariance-based discriminant model as a decision that must be made globally before fitting the classifier, Hetero-MoE allows that choice to become spatially adaptive and probabilistically represented. This creates a model in which prediction, representation, and local model selection are learned jointly. For pattern-recognition practitioners, this provides a practical mechanism for exploiting heterogeneous local geometry without abandoning interpretable predictive mechanisms; for researchers, it motivates a more specific set of unanswered questions concerning the statistical value, identifiability, selection, and interpretation of heterogeneous hypothesis classes within a common conditional mixture. In this sense, the significance of Hetero-MoE lies not in claiming that heterogeneous routing is new, but in demonstrating a concrete and interpretable formulation in which \emph{the inductive bias itself can become a spatially adaptive component of a probabilistic pattern-recognition model}.

\section*{Code Availability}

The official implementation of the proposed framework, including the source code, datasets, and experimental scripts, is publicly available at \href{https://github.com/ChatterjeeSoham/heterogenous-mixture-of-experts}
{GitHub Repository}

\bibliographystyle{elsarticle-num}
\bibliography{cas-refs-citation-order}

\end{document}